\documentclass[11pt,letterpaper]{article}

\def\draft{1} 
\def\doubleblind{1} 
\newcommand{\blind}[2]{{\ifnum\draft=1\color{purple}\fi \ifnum\doubleblind=1#2\fi\ifnum\doubleblind=0#1\fi\ifnum\doubleblind=2$\{$ #1 $\vert$ #2 $\}$\fi}}
\makeatletter
\renewcommand{\tableofcontents}{%
  \@starttoc{toc}%
}
\makeatother

\usepackage[margin=1in,verbose=true,letterpaper]{geometry}
\usepackage[utf8]{inputenc}
\usepackage[T1]{fontenc}
\usepackage{microtype}
\usepackage{amsthm}
\usepackage{amsthm, amsmath, amssymb, mathtools}
\usepackage{thmtools,thm-restate}

\usepackage{bm,bbm}
\usepackage{graphicx}
\usepackage{subcaption}
\usepackage{color,xcolor}
\usepackage{paralist,enumitem}
\usepackage{mathrsfs}
\usepackage[normalem]{ulem}
\usepackage{tabularx}
\usepackage{tikz}
\usetikzlibrary{shapes.geometric, arrows.meta, calc, backgrounds, positioning}
\usetikzlibrary{decorations.pathreplacing}
\usetikzlibrary{patterns}
\usepackage{pgfplots}
\pgfplotsset{compat=1.18}
\usepackage{amsmath, tikz-cd}
\usepackage{caption,comment}
\usepackage{booktabs}
\usepackage{algorithmicx}
\usepackage{algorithm}
\usepackage{algpseudocode}
\usepackage{fontawesome}
\usepackage[dvipsnames]{xcolor}
\usepackage[textsize=tiny]{todonotes}
\usepackage{standalone}
\usepackage{cancel}

\usepackage{parskip}

\renewcommand{\blind}[2]{{\ifnum\draft=1\color{purple}\fi \ifnum\doubleblind=1#2\fi\ifnum\doubleblind=0#1\fi\ifnum\doubleblind=2$\{$ #1 $\vert$ #2 $\}$\fi}}

\makeatletter
\newcommand{\algmargin}{\the\ALG@thistlm}
\algnewcommand{\parState}[1]{\State%
  \parbox[t]{\dimexpr\linewidth-\algmargin}{\strut #1\strut}}

\numberwithin{equation}{section}
\declaretheoremstyle[bodyfont=\it,qed=\qedsymbol]{noproofstyle}

\declaretheorem[name=Observation,numbered=no]{observation*}

\declaretheorem[numberlike=equation]{fact}

\declaretheorem[numberlike=equation]{theorem}

\declaretheorem[name=Theorem,numbered=no]{theorem*}

\declaretheorem[numberlike=equation]{lemma}
\declaretheorem[name=Lemma,numbered=no]{lemma*}

\declaretheorem[numberlike=equation]{corollary}
\declaretheorem[name=Corollary,numbered=no]{corollary*}

\declaretheorem[numberlike=equation]{proposition}
\declaretheorem[name=Proposition,numbered=no]{proposition*}

\declaretheorem[numberlike=equation]{claim}
\declaretheorem[name=Claim,numbered=no]{claim*}

\declaretheorem[name=Conjecture,numbered=no]{conjecture*}

\declaretheorem[name=Question,numbered=no]{question*}

\declaretheoremstyle[bodyfont=\it]{defstyle}

\declaretheorem[unnumbered,name=Definition,style=defstyle]{definition*}

\declaretheorem[unnumbered,name=Example,style=defstyle]{example*}

\declaretheorem[unnumbered,name=Notation=defstyle]{notation*}

\declaretheorem[numberwithin=section,style=defstyle]{assumption}
\declaretheorem[unnumbered,name=Assumption=defstyle]{assumption*}

\declaretheorem[unnumbered,name=Construction,style=defstyle]{construction*}

\declaretheoremstyle[]{rmkstyle} 

\newtheorem*{remark}{Remark}

\usepackage{xr-hyper}
\usepackage{hyperref}
\hypersetup{hypertexnames=false,colorlinks=true,citecolor=Aquamarine, linkcolor=magenta, urlcolor=blue}

\usepackage[capitalize,noabbrev,nameinlink]{cleveref}

\usepackage{arxiv}
\usepackage[style=alphabetic, backend=biber, minalphanames=3, maxalphanames=4, maxbibnames=99, maxcitenames=4]{biblatex}

\DeclarePairedDelimiter{\paren}{\lparen}{\rparen}

\DeclarePairedDelimiter{\ceil}{\lceil}{\rceil}

\DeclarePairedDelimiterX{\divx}[2]{(}{)}{#1 \;\delimsize\|\; #2}

\newcommand{\hare}{\ensuremath{\mathsf{UCB\text{-}HARE}}}

\newcommand{\Alg}{\ensuremath{\mathsf{A}}}

\newcommand{\mustar}{\mu_\star}
\newcommand{\muhat}{\hat{\mu}}
\newcommand{\Esg}[2]{\mathcal{E}_{\mathsf{SG}}(#1, #2)^{+}}
\newcommand{\Hist}{\mathcal{H}}

\newcommand{\Pow}{\ensuremath{\mathsf{M}}}

\newcommand{\mregret}{\mathfrak{R}}

\newcommand{\1}{\mathbbm{1}}
\newcommand{\E}{\mathbb{E}}
\newcommand{\Pbb}{\mathbb{P}}

\newcommand{\Gauss}{\mathcal{N}}
\newcommand{\Bern}{\mathsf{Ber}}
\newcommand{\Unif}{\mathsf{Unif}}

\newcommand{\KL}{\mathsf{KL}}
\newcommand{\kl}{\mathsf{kl}}

\newcommand{\calE}{\mathcal{E}}

\newcommand{\tbigO}{\widetilde{O}}

\DeclareMathOperator*{\argmax}{arg\,max}

\newenvironment{tightquote}
  {\begin{list}{}{\setlength{\leftmargin}{1em}\setlength{\rightmargin}{1em}}\item\itshape}
  {\end{list}}

\crefname{claim}{Claim}{Claims}
\Crefname{claim}{Claim}{Claims}
\crefname{fact}{Fact}{Facts}

\allowdisplaybreaks
\title{Price of Fairness in Bandits: A Tight Minimax Characterization}

\author{%
  Dhruv Sarkar\textsuperscript{1,2} \quad
  Soumyadeep Dutta\textsuperscript{1} \quad
  Sayak Ray Chowdhury\textsuperscript{3} \\[3pt]
  \normalfont\small\textsuperscript{1}Indian Institute of Technology Kharagpur \\[-1pt]
  \normalfont\small\textsuperscript{2}Mohamed bin Zayed University of Artificial Intelligence \\[-1pt]
  \normalfont\small\textsuperscript{3}Indian Institute of Technology Kanpur \\[2pt]
  \normalfont\small\texttt{dhruv.sarkar223@gmail.com} \\[1pt]
  \normalfont\small\texttt{soumyadeep.mind@gmail.com} \quad\texttt{sayakrc@iitk.ac.in} \\[2pt]
  }

\date{}
\begin{document}

\maketitle

\begin{abstract}
In sequential decision-making problems with bandit feedback, traditional algorithms typically minimize cumulative regret, treating early-round exploration as an acceptable amortized cost. However, in critical sequential settings (e.g., clinical trials), this utilitarian approach exposes early rounds (e.g., patients) to disproportionate and unfair ex-ante losses. To guarantee equitable outcomes in each round, a recent line of work evaluates the sequence of per-round expected rewards through the generalized $p$-mean function, which interpolates between utilitarian welfare ($p=1$), Nash social welfare ($p\to 0$) and Rawlsian fairness ($p\to-\infty$). While the regime $p\ge 0$ is by now well understood, with tight upper and lower bounds, the \emph{strictly fair} regime, parameterized by $q=-p>0$, is not.

The difficulty stems from the inverse-power form of the objective: a negative-power mean is dominated by the smallest per-round expected reward, so a policy that assigns very low expected reward to even one round incurs a welfare loss that cannot be offset by later rounds. To sidestep this, existing algorithms rely on uniform exploration in early rounds. For $\sigma$-sub-Gaussian rewards with non-negative means,
the best known algorithm~\cite{sarkar2025welfarist} incurs regret of order $O\big(k^{(q+1)/2}/\sqrt{T}\big)$, where $k$ is the number of arms and $T$ is the total number of rounds. At the same time, the known lower bound in this case is the classical average-regret limit $\Omega\big(\sigma\sqrt{k/T}\big)$, which does not capture the hardness of the problem when $q >0$. The price of \emph{strict fairness} is therefore open on \emph{both} sides: prior upper bound held only against a loose lower bound, leaving unresolved how much of the
$k^{(q+1)/2}$ dependence is information-theoretically necessary, and how much is due to uniform exploration.

We close this gap by identifying the
exact polynomial penalty on the number of arms as $k^{q/2}$. First, using a needle-in-haystack construction, we prove an algorithm-independent lower bound of
$ \Omega\big(\sigma\sqrt{k^{\max(1,q)}/T}\big)$; for $q>1$ this lifts the classical
$\sqrt{k/T}$ rate to $\sqrt{k^{q}/T}$, showing that the $k^{q/2}$ penalty is information-theoretically unavoidable. Second, we introduce the
\textsf{UCB-HARE} (Harmonic Anchored Rank Exploration) algorithm, which replaces uniform sampling
with an inverse-weighted rank schedule shielded by a certified positive-mean anchor. 
We show that its leading-order regret is $\widetilde{O}\big(\sigma \sqrt{k^{\max(1,q)}/T}\big)$, which
matches the lower bound up to logarithmic factors.
Experiments on synthetic instances confirm that \textsf{UCB-HARE} outperforms the uniform-exploration-based baselines, with
the advantage growing in $q$.

\end{abstract}

\thispagestyle{empty}

\newpage 


\pagenumbering{arabic} 
\addtocontents{toc}{\protect\setcounter{tocdepth}{-1}}
\section{Introduction}


The stochastic multi-armed bandit problem models sequential decision-making under
uncertainty over a horizon of $T$ rounds. A learner faces $k$ distinct actions (arms),
each associated with an unknown probability distribution $\nu_i$ supported on $\mathbb{R}$
with mean $\mu_i$. We denote the problem instance by the vector of distributions
$\nu=(\nu_1,\dots,\nu_k)$ and let $\mu_\star:=\max_{i\in[k]}\mu_i$ denote the optimal mean.
At each round $t\in[T]$, the learner selects an action $I_t\in[k]$ according to a policy (an algorithm)
$\Alg$, which maps the observable history of past actions and rewards to a probability distribution
over $[k]:=\{1,\ldots,k\}$, and receives an independent reward $X_t\sim\nu_{I_t}$. We define the ex-ante reward at round \(t\) as $m_t:=\mathbb E[\mu_{I_t}]$, where the expectation is over both the reward draws and the algorithm's internal randomization (if any).
The learner is evaluated via the arithmetic mean of this sequence, yielding the standard average regret $R_{T, \textsf{avg}} 
    (\Alg, \nu) := \mustar - \frac{1}{T}\sum_{t=1}^T m_t$ \cite{bubeck2012regret}.
    
This utilitarian objective maximizes total payoff, inherently treating early-round exploration as an acceptable amortized cost. 
When rounds correspond to distinct
individuals, such as patients in a clinical trial \cite{Thompson1933}, this amortization
inflicts welfare losses on early participants: an algorithm may achieve asymptotically optimal average regret even after assigning zero expected reward to a sizeable fraction of initial users. To guarantee equitable treatment across all users, recent works embed social
welfare functions directly into the learning objective. Grounded in the classical axioms
of collective welfare, most notably inequality aversion via the Pigou-Dalton transfer
principle \cite{moulin2004fair}, a principled replacement for the arithmetic mean is the
generalized $p$-mean. For an ex-ante reward sequence $(m_1,\ldots,m_T)$, with each $m_t \in \mathbb{R}_{\ge 0}$, it is defined as 
\[
  \Pow_p(m_1,\dots,m_T)\ :=\
  \begin{cases}
    \big(\tfrac1T\sum_{t=1}^T (m_t)^p\big)^{1/p}, & p\neq 0,\\[4pt]
    \big(\prod_{t=1}^T m_t\big)^{1/T}, & p=0,
  \end{cases}
\]
with corresponding $p$-mean regret 
$R_{T,p}(\Alg,\nu):=\mu^\star-\Pow_p(m_1,\dots,m_T)$.
The parameter $p$ controls the strictness of the fairness requirement: $p=1$ recovers
classical average regret, $p\to0$ maximizes the geometric mean (Nash social welfare),
and $p\to-\infty$ approaches the strict Rawlsian maximin ideal.
Throughout the paper, we make the following assumption.

\begin{assumption}[Non-Negative Means]\label{}
The expected rewards satisfy $\mu_i \ge 0$ for all $i \in [k]$, with the optimal arm satisfying $\mu_\star > 0$.
\end{assumption}

This assumption is standard \cite{barman2022fairness, krishna2025pmean, sarkar2025welfarist} in social welfare applications such as clinical trials, where arms are pre-screened for baseline safety, and the goal is to identify the optimal arm from a \textit{non-harmful} pool. Furthermore, fairness metrics derived from the geometric mean (e.g., Nash Social Welfare \cite{barman2022fairness}) are only well-defined over non-negative expected values. Note that bounding the \emph{expected} reward does not preclude negative \emph{individual} outcomes: a treatment which is beneficial on average ($\mu_i \ge 0$) may still induce adverse side effects in specific instances. 


In this work, we study the fundamental limits of the \emph{strictly fair} regime,
parameterized by $q=-p>0$. By convention, if $m_t=0$ for any $t\in[T]$, we set
$\Pow_{-q}(m_1,\dots,m_T)=0$. The challenge of this regime is inverse-power: because negative power means convert small expected rewards into large penalties, the welfare is bottlenecked by the worst single round,
$\Pow_{-q}(m_1,\dots,m_T)\ \le\ T^{1/q}\,\min_{t\in[T]} m_t$.
Consequently, any exploration policy that assigns very low probability to the optimal arm, by driving $m_t$
toward zero for even a small fraction of rounds, faces an unrecoverable welfare collapse.

\subsection{Prior Work and Gaps}

A recent line of work has studied whether classical bandit methods can be adapted to
welfare objectives stronger than average reward. \cite{barman2022fairness} initiated this
direction for Nash regret ($p\to0$), and \cite{krishna2025pmean} extended the framework
to generalized $p$-mean regret for arbitrary $p$.
For the sub-Gaussian model, \cite{sarkar2025welfarist} proposed \textsf{Welfarist-UCB},
which pairs a data-adaptive uniform exploration phase with a subsequent UCB phase -- the
standard optimistic rule that pulls the arm maximizing \(\widehat\mu_i(t)+c_i(t)\), the
empirical mean plus a confidence radius, thereby favoring arms that appear promising or remain
uncertain \cite{bubeck2012regret}. Uniform exploration ensures that all arms have sufficiently many samples before this optimistic rule is
used. This is important for negative-power welfare: rounds with very small ex-ante reward
contribute \((m_t)^{-q}\) inside the objective and can dominate the regret. Writing
\(q=-p>0\), their algorithm attains $\widetilde{O}_q\!\left(\frac{\sigma\,k^{(q+1)/2}}{\sqrt{T}}\right)$
regret.\footnote{A subscript \(q\) in asymptotic notation
(e.g.\ \(\widetilde{O}_q\), \(\Omega_q\), \(\Theta_q\)) indicates that hidden constants may
depend on \(q\).}

This guarantee leaves open the optimal dependence on the number of arms, especially in
the strongly fair regime $q>1$. Uniform exploration assigns the unknown optimal arm only
a $1/k$ share of early pulls. For negative powers, such low selection probabilities are
costly: if the optimal arm is the only arm with a positive mean and is selected with
probability $a_t$ at round $t$, then the corresponding inverse-power contribution scales
as $a_t^{-q}$. Thus, the factor $k^{(q+1)/2}$ may reflect both genuine statistical
difficulty and the inefficiency of exploring all arms symmetrically.

Before this work, the main generally applicable lower bound was the classical
average-regret lower bound
$\Omega\!\left(\sigma\sqrt{\frac{k}{T}}\right)$ \cite{bubeck2012regret},
which extends to all $p\le1$ by monotonicity of power means. This bound provides an
important baseline, but it is not tailored to negative-power welfare. Average regret
measures additive losses over time, whereas a negative-power mean can be dominated by
rounds in which the ex-ante reward is unusually small. As a result, existing lower bounds
did not determine whether the additional dependence on $k$ incurred by uniform
exploration is unavoidable.
This leaves the following question:
\begin{tightquote}
  \itshape
  What is the intrinsic polynomial dependence on $k$ for negative-power welfare, and how
  much of the $k^{(q+1)/2}$ dependence of uniform exploration is information-theoretically
  necessary?
\end{tightquote}

Resolving this question, especially if the minimax scale remains $\sqrt{k/T}$ or rises to $k^{(q+1)/2}/\sqrt{T}$, requires both a sharper lower bound that
captures the inverse-power sensitivity to low-reward rounds, and a more efficient algorithm whose exploration rule is not constrained to be uniform.

\subsection{Our Contributions}

Throughout, we treat the fairness level $q = -p \ge 0$ as fixed and characterize the dependence of
the regret on the number of arms $k$ and time horizon $T$. We establish that the minimax rate is $k^{q/2}/\sqrt{T}$ by first proving an algorithm-independent lower bound and then designing an algorithm that achieves it up to poly-logarithmic factors.


\textbf{1. The Fundamental Price of Strong Fairness.} 
First, we show the polynomial dependence on $k$ is an information-theoretic necessity, and not due to uniform exploration. Suppose each arm $i \in [k]$ has a $\sigma$-sub-Gaussian reward distribution $\nu_i$ with a non-negative mean $\mu_i \ge 0$, i.e., $\mathbb{E}_{X \sim \nu_i}[\exp(\lambda(X - \mu_i))] 
    \le \exp\left(\frac{\sigma^2 \lambda^2}{2}\right)$ for all $\lambda \in \mathbb{R}$.
Define the class
\[
\Esg{k}{\sigma} := 
\left\{ \nu = (\nu_1, \dots, \nu_k) \mid \forall i \in [k],\, 
\nu_i \text{ is } \sigma\text{-sub-Gaussian with mean } 0 \le \mu_i < \infty \right\}.
\]

The minimax $(-q)$-mean regret over this class is defined as 
\begin{align*}
   \mathfrak{R}_{T,-q}(k,\sigma)\ :=\ \inf_{\Alg}\ \sup_{\nu\in \Esg{k}{\sigma}}\
  R_{T,-q}(\Alg,\nu),\quad \text{where} \quad R_{T,p}(\Alg,\nu):=\mu^\star-\Pow_p(m_1,\dots,m_T).
\end{align*}
We focus on the case that each $\mu_i$ and $\sigma$ are constants that do not scale
with $k$ and $T$. 


\begin{theorem}[Minimax lower bound]
\label{thm:lower-combined}
Fix $q=-p>0$. For any number of arms $k\ge 8$ and horizon $T\ge 1$, the $(-q)$-mean regret over $\Esg{k}{\sigma}$ satisfies
\[
  \mathfrak{R}_{T,-q}(k,\sigma)\ \ge\ \Omega_q\!\left(\sigma\sqrt{\frac{k^{\max(1,q)}}{T}}\right).
\]
For the strongly fair regime ($q>1$), this lifts the classical $\sqrt{k/T}$ bound to
exactly $k^{q/2}/\sqrt{T}$.
\end{theorem}
We prove this lower bound by considering a class of "needle-in-haystack" type bandit instances that are difficult to distinguish. While our high-level approach is similar to existing regret lower-bound proof techniques for average welfare~\cite{bubeck2012regret}, additional challenges arise due to negative-power welfare. Specifically, we show how it 
penalizes low selection probabilities. Since the learner rarely pulls the optimal arm at the start, this keeps the ex-ante reward low for many rounds, which is exactly what the negative-power objective penalizes, eventually forcing the lower bound.

Crucially, the lower bound is not specific to negative power means. It follows from a
general \emph{information-cost inequality} (Lemma~\ref{lem:info-cost}) that lower-bounds the
cumulative cost of any fairness-sensitive objective whose welfare is bottlenecked by the rate
at which a learner can identify the optimum. We instantiate it here for the $p$-mean
welfare, where it yields the $k^{q/2}$ price; the same principle is can be adapted to general bandit settings, e.g., linear bandits~\cite{abbasi2011improved,sarkar2026improvedalgorithmsnashwelfare}.

\textbf{2. Matching the Limit via \textsf{UCB-HARE} Algorithm.}
Our analysis reveals that uniform exploration is suboptimal in its dependence on $k$: by locking the
optimal arm's selection probability at $1/k$, one keeps its selection probability small for an extended
duration, thereby incurring an avoidable welfare penalty. We propose
\textsf{UCB-HARE} (Harmonic Anchored Rank Exploration), which replaces uniform sampling with "shielded" exploration: it explores a random permutation of the arms via a
harmonically decaying schedule, allocating pulls proportionally to $1/r$ (where $r$ is an
arm's rank) rather than uniformly. This schedule guarantees the rapid discovery of
an intermediate arm with a strictly positive lower confidence bound: our \emph{anchor}.

The core insight is that defending the welfare objective does not require waiting for the global optimum; the anchor provides a sufficient safety net. Once discovered, the algorithm couples every risky exploration pull in the tail of the permutation with a safe exploitation pull of the anchor. This dynamic pairing shields the ex-ante reward sequence, allowing the algorithm to safely continue its search for the optimal arm without ever dropping below a critical welfare baseline.

\begin{theorem}[Instance-Dependent Upper Bound]\label{thm:intro-upper}
Fix $q = -p > 0$. For any fixed instance $\nu \in \Esg{k}{\sigma}$ with optimal mean $\mustar > 0$, \textsf{UCB-HARE} guarantees
\[
    R_{T,-q}(\textsf{UCB-HARE}, \nu) 
    \le 
    \tbigO_q\paren*{ \sigma\sqrt{\frac{H_k K_q(k)}{T}} + \frac{\mustar H_k K_q(k)}{T} },
\]
where $H_k = \sum_{r=1}^k 1/r$ is harmonic number, and $K_q(k) = \sum_{r=1}^k (k/r)^q$ is generalized harmonic sum. 
\end{theorem}
By the standard asymptotics \(H_k=\Theta(\log k)\) and
\[
K_q(k)
=
\begin{cases}
    \Theta_q(k), & 0<q<1,\\
    \Theta(k\log k), & q=1,\\
    \Theta_q(k^q), & q>1.
\end{cases}
\implies
R_{T,-q}(\hare,\nu)
\le
\begin{cases}
    \widetilde O_q\!\left(\sigma\sqrt{k/T}\right), & 0<q<1,\\[2mm]
    \widetilde O\!\left(\sigma\sqrt{k/T}\right), & q=1,\\[2mm]
    \widetilde O_q\!\left(\sigma\sqrt{k^q/T}\right), & q>1.
\end{cases}
\]
Thus, after poly-logarithmic factors are suppressed, \hare\ matches the lower bound
for every fixed \(q>0\): for \(0<q\le1\), the classical \(\sqrt{k/T}\) term
remains dominant, while for \(q>1\) the negative-power objective raises the
polynomial dependence to \(k^{q/2}/\sqrt T\). 

\textbf{Comparison with Prior Work.} The best known upper bound for $(-q)$-mean regret is $\widetilde{O}_q\!\left(\frac{\sigma\,k^{(q+1)/2}}{\sqrt{T}}\right)$, which is achieved by the \textsf{Welfarist-UCB} algorithm \cite{sarkar2025welfarist}. In contrast, \hare\ achieve the minimax $\widetilde{O}_q\!\left(\frac{\sigma\,k^{q/2}}{\sqrt{T}}\right)$
regret. Table~\ref{tab:rates} summarizes the comparison.



\begin{table}[htbp]
    \centering
    \caption{Regret bounds for $q=-p>0$ (polylog factors, $q$-dependent constants suppressed). For $q>1$, our bounds raise the $\sqrt{k/T}$ rate to $\sqrt{k^{q}/T}$, closing the gap left open by prior work.}  
    \label{tab:rates}
    \begin{tabular}{@{}lcccc@{}}
        \toprule
        \textbf{Regime} & \textbf{Prior Lower Bound} & \textbf{Our Lower Bound} & \textbf{\textsf{Welfarist-UCB}} & \textbf{\textsf{UCB-HARE}(Ours)} \\
        \midrule
        $0 < q < 1$ & $\sigma\sqrt{k/T}$ & $\sigma\sqrt{k/T}$ & $\sigma k^{(q+1)/2}/\sqrt{T}$ & $\sigma\sqrt{k/T}$ \\[4pt]
        $q = 1$     & $\sigma\sqrt{k/T}$ & $\sigma\sqrt{k/T}$ & $\sigma k/\sqrt{T}$           & $\sigma\sqrt{k/T}$ \\[4pt]
        $q > 1$     & $\sigma\sqrt{k/T}$ & $\sigma k^{q/2}/\sqrt{T}$ & $\sigma k^{(q+1)/2}/\sqrt{T}$ & $\sigma k^{q/2}/\sqrt{T}$ \\
        \bottomrule
    \end{tabular}
\end{table}



\subsection{Related Work: Generalized-Mean Welfare Beyond Bandits}
\label{sec:related-nash-pmean}

Nash and generalized-mean welfare are standard objectives in the algorithmic study
of fair allocation. The Nash objective -- the geometric mean \(\left(\prod_i
u_i\right)^{1/n}\), equivalently \(\sum_i\log u_i\) -- originates in Nash's axiomatic
treatment of bargaining \cite{nash1950bargaining}. The generalized-mean family
\(\Pow_p(u_1,\ldots,u_n)=\left(\frac{1}{n}\sum_{i=1}^n u_i^p\right)^{1/p}\) embeds it in a
continuum: \(p=1\) is utilitarian welfare, \(p\to0\) is Nash welfare, and \(p\to-\infty\) is
egalitarian max-min welfare, with smaller \(p\) encoding stronger aversion to unequal
utility profiles \cite{atkinson1970measurement,moulin2004fair}. The objective studied here
is thus not specific to bandit learning, but the standard parametric language for
interpolating between efficiency and fairness.

In divisible allocation and market models, the
Eisenberg--Gale convex program maximizes logarithmic utilities, linking Fisher-market
equilibria to Nash-social-welfare maximization via convex duality
\cite{eisenberg1959consensus,cole2017convex}. The same logarithmic objective underlies
proportional fairness in network resource allocation, and the broader \(\alpha\)-fair family
interpolates between proportional and max-min fairness, which are equivalent, up to monotone
transformations, to the generalized-mean~\cite{kelly1998rate,mo2000fair,lan2010axiomatic}.
In indivisible fair division, maximum Nash welfare is canonical because, for additive
valuations, it simultaneously guarantees Pareto optimality and envy-freeness up to one good
\cite{caragiannis2019unreasonable}, spawning a substantial approximation-algorithms
literature \cite{cole2015nash,anari2017nash}. Closest to our parameterized objective, recent
work studies \(p\)-mean welfare in fair division and online allocation across the full range
from utilitarian to Nash to egalitarian welfare
\cite{barman2020uniform,barman2020pmean,barman2022generalizedmean,eckart2024normalizedpmeans}.

The distinction in the bandit setting is that the welfare vector consists of per-round ex-ante
rewards generated by a learning algorithm, rather than utilities induced by a known
allocation. The question is therefore not merely how to optimize an
inequality-sensitive objective, but how much additional cost is incurred to protect that objective while the learner is still acquiring information about arms.

\section{Technical Overview}\label{sec:tech-overview}


We unfold our technical approach in three stages. The first establishes the fundamental limits of exploration in the strongly fair regime. Second, we present our algorithm, \hare. Finally, its upper bound analysis shows how our method avoids prior inefficiencies to achieve the optimal rate.

\subsection{Lower Bound}\label{subsec:overview-lower}



We obtain a lower bound on the minimax regret by studying a family of instances instead of a single~one. Suppose the true instance $\nu$ is drawn uniformly at random from a finite set
\(\mathcal V := \{\nu^{(1)}, \ldots, \nu^{(k)}\}\). If we can lower-bound the regret of an
arbitrary algorithm averaged over $j \in \{1, \ldots, k\}$, it follows that there must exist at least one specific index $j$ (and hence, an instance $\nu^{(j)}$) for which the same lower bound applies. We construct $\mathcal V$ such that any arm $i$ is optimal for at most one instance in this set -- a ``needle-in-haystack'' configuration.

Fix some \(\mu > 0\). For each $j \in [k]$, we define the instance \(\nu^{(j)} = \bigl(\nu_1^{(j)}, \ldots, \nu_k^{(j)}\bigr)\) by
\[
  \nu_j^{(j)} = \mathcal N(\mu, \sigma^2) \quad \text{and} \quad \nu_i^{(j)} = \mathcal N(0, \sigma^2) \text{ for } i \neq j.
\]
so that the optimal arm for $\nu^{(j)}$ (i.e., the $j$-th arm) fails to be optimal for any other instance $\nu^{(j')}$. We write \(\mathbb P_j\) for the
law under \(\nu^{(j)}\) and \(\mathbb P(\cdot)= \frac{1}{k}\sum_{j=1}^k \mathbb P_j(\cdot)\) be the law averaged over a uniformly random instance index. Both probability measures account for reward randomness and the algorithm's internal randomization, if any. Since
\(\mathcal V\subseteq\Esg{k}{\sigma}\), any lower bound for \(\mathcal V\) also holds over \(\Esg{k}{\sigma}\).

Fix an arbitrary algorithm. Let
\(a_{j,t}:=\mathbb P_j(I_t=j)\) be the probability of pulling the optimal arm at round \(t\)
under \(\nu^{(j)}\).
Since all other arms have mean zero, the ex-ante reward on \(\nu^{(j)}\) is
\(m_t^{(j)}=\mu\,a_{j,t}\), so the negative-power welfare on \(\nu^{(j)}\) depends only on
the sequence $(a_{j,t})_t$. Let
\(s_t:=\frac{1}{k}\sum_{j=1}^k a_{j,t}\) denote the average success
probability at round \(t\). Since \(x\mapsto x^{-q}\) is convex for \(q=-p>0\), Jensen's inequality gives
\(\frac{1}{k}\sum_{j} a_{j,t}^{-q}\ge s_t^{-q}\) for every \(t\); averaging over time and
then over the \(k\) instances, some \(\nu^{(j)}\in\mathcal V\) satisfies
\(\frac{1}{T}\sum_t a_{j,t}^{-q}\ge\frac{1}{T}\sum_t s_t^{-q}\). Since \(x\mapsto x^{-1/q}\)
is decreasing, this yields \((-q)\)-mean welfare of the ex-ante reward sequence
\[
  \Pow_{-q}\bigl(m_1^{(j)},\ldots,m_T^{(j)}\bigr):=\paren*{\frac{1}{T}\sum\nolimits_{t=1}^T (m_t)^{-q}}^{-1/q}
  =\mu\paren*{\frac{1}{T}\sum\nolimits_{t=1}^T a_{j,t}^{-q}}^{-1/q}
  \le\mu\paren*{\frac{1}{T}\sum\nolimits_{t=1}^T s_t^{-q}}^{-1/q}.
\]
The lower bound, therefore, reduces to the question: how large an inverse-power cost \(\sum_t s_t^{-q}\) must a learner incur to
raise \(s_t\) from its uninformed baseline \(1/k\) to a constant scale? 

Two key estimates connect the success probability \(s_t\) to the history $\Hist_{t-1}= \{I_1,X_1,\ldots, I_{t-1},X_{t-1}\}$. Let
\(\mathcal I_t := I(J;\Hist_{t-1})\) denote the mutual information that the history carries about
the arm $J \sim \mathrm{Unif}([k])$.\footnote{Three objects share the symbol \(I\) and should not
be conflated: the \emph{action} \(I_t\in[k]\) selected at round \(t\); the
mutual-information operator \(I(\cdot\,;\cdot)\); and the mutual information \(\mathcal I_t:=I(J;\Hist_{t-1})\) between
the random index and the history. We write \(\mathcal I_t\) in script throughout to
distinguish it from the action \(I_t\).} For \(s\in[1/k,1]\), define
$d_k(s):=\kl\!\left(s\,\middle\|\,\tfrac1k\right)$,
the binary relative entropy between Bernoulli\((s)\) and Bernoulli\((1/k)\). It satisfies
\(d_k(1/k)=0\), is nonnegative, and is increasing on \([1/k,1]\). Because \(s_t\) may dip
below \(1/k\), we track the truncated quantity \(r_t:=\max\{s_t,1/k\}\), so that
\(d_k(r_t)\) is always evaluated on its intended domain.

The first estimate says that a success probability above the random-guessing scale can be
achieved only if the history distinguishes the identity of the positive-mean arm.
\begin{lemma}[Informal, see \Cref{lem:success-info}]
\label{lem:informal-success-evidence}
For every round \(t\), \(\ d_k(r_t)\le\mathcal I_t\).
\end{lemma}
To see this, note that if \(s_t<1/k\), then \(r_t=1/k\) and the claim follows from \(d_k(1/k)=0\) and
nonnegativity of mutual information. Otherwise \(s_t\ge 1/k\), so \(r_t=s_t\). Here, the
action \(I_t\) can be viewed as an estimator of the hidden index \(J\). Since $I_t$ is
generated from \(\Hist_{t-1}\) and is independent of \(J\), the
data-processing inequality gives \(I(J;I_t)\le I(J;\Hist_{t-1})=\mathcal I_t\), while this
estimator succeeds with probability \(\mathbb P(I_t=J)=s_t\), so the Fano converse yields
\(I(J;I_t)\ge d_k(s_t)\). Combining the two, we obtain \(\mathcal I_t\ge d_k(s_t)=d_k(r_t)\). \qed

The second estimate controls the one-step growth of \(\mathcal I_t\) in the "needle-in-haystack" problem.
\begin{lemma}[Informal, see \Cref{lem:info-gain}]
\label{lem:informal-one-step-growth}
At every round \(t\), $\mathcal I_{t+1}-\mathcal I_t\le \lambda s_t\le \lambda r_t$, where \(\lambda:=\KL\!\left(\mathcal N(\mu,\sigma^2)\,\middle\|\,\mathcal N(0,\sigma^2)\right)
=\mu^2/(2\sigma^2)\).
\end{lemma}
By the chain rule, the one-step increase equals the conditional mutual information between
\(J\) and the new reward, given the past history and the chosen action. Compare the
conditional reward law to the reference \(\mathcal N(0,\sigma^2)\): if \(I_t\neq J\) the reward has law \(\mathcal N(0,\sigma^2)\), so its divergence from the reference is zero,
whereas if \(I_t=J\) the reward has law \(\mathcal N(\mu,\sigma^2)\), at KL-divergence exactly
\(\lambda\). Averaging over \(\{I_t=J\}\) gives
\(\mathcal I_{t+1}-\mathcal I_t\le\lambda\,\mathbb P(I_t=J)=\lambda s_t\); the second
inequality follows from \(s_t\le r_t\). \qed

The two estimates above reduce the lower-bound argument to a deterministic calculation. The
variable \(r_t\) tracks the learner's progress from the random-guessing scale \(1/k\) toward
constant success probability. When \(r_t\) is of order \(r\), the inverse-power penalty is
of order \(r^{-q}-1\), while one round can acquire at most \(\lambda r\) units of information.
Raising \(r_t\) from \(1/k\) to a constant scale thus costs \(r^{-q-1}/\lambda\) per unit of
information demanded, and integrating this rate against the information requirement
\(d_k'(s)\) gives a leading cost of order
$\frac{1}{\lambda}\int_{1/k}^{1/2} s^{-q-1}\,d_k'(s)\,ds$.
Since \(d_k'(s)=\log\!\left(\dfrac{(k-1)s}{1-s}\right)\), this integral already contributes
order \(k^q\) on the subinterval \([2/k,4/k]\), yielding the following deterministic
inequality.
\begin{claim}[Informal, see \Cref{lem:det-cost}]
\label{claim:informal-cost}
Let \(r_t\in[1/k,1]\) and let \(\mathcal I_t\) be nondecreasing with $\mathcal I_1=0,\,d_k(r_t)\le\mathcal I_t$, and $\mathcal I_{t+1}-\mathcal I_t\le\lambda r_t$ for all $t\in[T]$.
Then there exists a constant \(c_q>0\), depending only on \(q\), such that
$\sum_{t=1}^T\paren*{r_t^{-q}-1}\ \ge\ c_q\,\min\paren*{T,\ \frac{k^q}{\lambda}}$.
\end{claim}
Since \(r_t\ge s_t\) gives \(s_t^{-q}\ge r_t^{-q}\), combining the averaging argument above
with \Cref{claim:informal-cost} 
produces a fixed instance \(\nu^{(j)}\) with
$\frac{1}{T}\sum_{t=1}^T a_{j,t}^{-q}\ \ge\ 1+\Omega_q\paren*{\min\paren*{1,\ \frac{k^q}{\lambda T}}}$.
Substituting \(\lambda=\mu^2/(2\sigma^2)\), the \((-q)\)-mean regret is bounded below by
\(\Omega_q\paren*{\mu\,\min(1, \frac{\sigma^2k^q}{\mu^2 T})}\). Choosing
\(\mu=\sigma\sqrt{k^q/T}\) balances the two terms and yields the negative-power lower bound
\(\Omega_q \paren*{\sigma k^{q/2}/\sqrt{T}}\). Combined with the classical average-regret
baseline \(\Omega\paren*{\sigma\sqrt{k/T}}\) \cite{bubeck2012regret} (which carries over to \((-q)\)-mean regret by monotonicity of generalized means (\Cref{lem:power-monotone})), this gives the minimax lower bound
\(\Omega_q\paren*{\sigma\sqrt{k^{\max\{1,q\}}/T}}\). \qed

\textbf{A general information-cost inequality.}
Claim~\ref{claim:informal-cost} is a special case of a general inequality that makes no reference to the bandit objective. Its proof uses only three features: the per-round penalty \(r^{-q}-1\) is a non-increasing function of
the success probability; the information needed to sustain success probability \(r\) is at
least \(d_k(r)\); and this information can grow by at most \(\lambda r\) per round. Abstracting
these into a cost \(c\), an information requirement \(\phi\), and an acquisition rate
\(\psi\), the same argument gives the following (\Cref{lem:info-cost}): if
\(\phi(r_t)\le\mathcal I_t\) and \(\mathcal I_{t+1}-\mathcal I_t\le\lambda\,\psi(r_t)\) for a
non-increasing cost \(c\), a strictly increasing \(\phi\) with \(\phi(1/k)=0\), and a
non-decreasing \(\psi\), then
\[
  \sum_{t=1}^T c(r_t)\ \ge\ \min\!\left\{\,c(s^\star)\,T,\ \
  \frac1\lambda\int_{1/k}^{s^\star}\frac{c(r)\,\phi'(r)}{\psi(r)}\,\mathrm dr\right\}
\]
for any target scale \(s^\star\). The negative-power bound of Claim~\ref{claim:informal-cost}
is recovered (\Cref{cor:neg-power}) by taking \(c(r)=r^{-q}-1\), \(\phi=d_k\), \(\psi(r)=r\),
and \(s^\star=1/2\), whereupon the integral evaluates to order \(k^q/\lambda\). Stated this way, the objective enters only through the cost \(c\) and the problem's structure
only through the acquisition rate \(\psi\), so the inequality applies well beyond negative
power means. 


\subsection{The Algorithm}\label{subsec:overview-algo}


Our lower bound exposes that regret depends on how fast the optimal arm's pull probability rises above the baseline $1/k$. Yet prior algorithms \cite{sarkar2025welfarist} rely on \textit{reward-agnostic} uniform exploration -- pulling
every arm, including the optimal one, at the same low rate while gathering evidence against the rest. 
Our key is a more judicious initial exploration, designed to
quickly secure a nontrivial \textit{growing} lower bound on the ex-ante reward. 
We achieve this via a two-phase approach:

\textbf{Phase I: Preparation.} 
In this phase, we perform \emph{shielded} exploration of arms by pairing each exploration pull with a safer \emph{anchor} pull. 
For this, first \emph{rank} these arms: draw a uniform random permutation $\pi$ over $[k]$, mapping each rank $r \in [k]$ to a unique arm $\pi_r$.
Next, partition the time horizon into a rigid tiling of 2-round blocks (see \Cref{fig:tiled_horizon}).
Each block $b \ge 1$ spans the two bandit rounds: $t=2b-1$ and $t=2b$. 
Of these, the algorithm reserves one round (the \textit{Scheduled Slot}) 
to explore arms,
and the other round (the \textit{Auxiliary Slot}) to exploit the safest known arm (the \emph{anchor}).


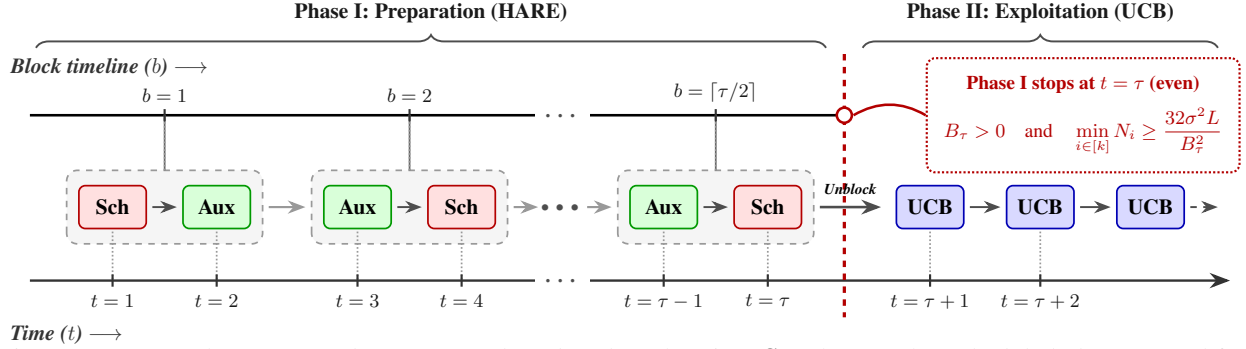
\begin{figure*}[!ht]
\centering
\resizebox{\textwidth}{!}{
\begin{tikzpicture}[x=1cm, y=1cm, >=Stealth, font=\sffamily\small]

    \tikzset{
        sch/.style={rectangle, draw=red!70!black, fill=red!12, line width=0.8pt, rounded corners=3pt, minimum width=1.1cm, minimum height=0.7cm, font=\bfseries, align=center},
        anc/.style={rectangle, draw=green!60!black, fill=green!15, line width=0.8pt, rounded corners=3pt, minimum width=1.1cm, minimum height=0.7cm, font=\bfseries, align=center},
        ucb/.style={rectangle, draw=blue!70!black, fill=blue!15, line width=0.8pt, rounded corners=3pt, minimum width=1.1cm, minimum height=0.7cm, font=\bfseries, align=center},
        blockbox/.style={rectangle, draw=black!40, dashed, line width=0.8pt, fill=black!4, rounded corners=5pt},
        flow/.style={-{Stealth[scale=1.1]}, line width=0.9pt, draw=black!70},
        interblock/.style={-{Stealth[scale=1.1]}, line width=1pt, draw=black!40},
        axis/.style={-{Stealth[scale=1.2]}, line width=1.1pt, draw=black!80}
    }

    \draw [decorate, decoration={brace, amplitude=6pt}, line width=0.9pt, draw=black!70] 
        (-0.1, 2.6) -- (12.8, 2.6) node [midway, above=8pt, font=\bfseries\normalsize, text=black!90] {Phase I: Preparation (HARE) };
    \draw [decorate, decoration={brace, amplitude=6pt}, line width=0.9pt, draw=black!70] 
        (13.4, 2.6) -- (19.2, 2.6) node [midway, above=8pt, font=\bfseries\normalsize, text=black!90] {Phase II: Exploitation (UCB)};

    \draw[line width=1.2pt] (-0.2, 1.5) -- (13.1, 1.5);
    \node[above, font=\bfseries, text=black!80] at (1.10, 2.0) {\textit{Block timeline ($b$) $\longrightarrow$}};

    \draw[axis] (-0.2, -1.2) -- (19.4, -1.2);
    \node[below, font=\bfseries, text=black!80] at (0.4, -1.8) {\textit{Time ($t$) $\longrightarrow$}};

    \draw[blockbox] (0.4, -0.6) rectangle (3.6, 0.6);
    \node[sch] (s1) at (1.15, 0) {Sch}; 
    \node[anc] (a1) at (2.85, 0) {Aux};
    \draw[flow] (1.8, 0) -- (2.2, 0);
    \draw[line width=0.8pt, draw=black!60] (2.0, 1.5) -- (2.0, 0.6);
    
    \draw[blockbox] (4.4, -0.6) rectangle (7.6, 0.6);
    \node[anc] (a2) at (5.15, 0) {Aux}; 
    \node[sch] (s2) at (6.85, 0) {Sch};
    \draw[flow] (5.8, 0) -- (6.2, 0);
    \draw[line width=0.8pt, draw=black!60] (6.0, 1.5) -- (6.0, 0.6);

    \draw[interblock] (3.7, 0) -- (4.3, 0);
    \draw[interblock] (7.7, 0) -- (8.1, 0);
    
    \node[font=\LARGE\bfseries, text=black!70] at (8.5, 0) {\dots};
    \node[fill=white, inner sep=4pt, font=\Large, text=black!70] at (8.5, 1.5) {\dots};
    \node[fill=white, inner sep=4pt, font=\Large, text=black!70] at (8.5, -1.2) {\dots};
    \draw[interblock] (8.9, 0) -- (9.3, 0);

    \draw[blockbox] (9.4, -0.6) rectangle (12.6, 0.6);
    \node[anc] (ab) at (10.15, 0) {Aux}; 
    \node[sch] (sb) at (11.85, 0) {Sch};
    \draw[flow] (10.8, 0) -- (11.2, 0);
    \draw[line width=0.8pt, draw=black!60] (11.0, 1.5) -- (11.0, 0.6);

    \draw[line width=1.5pt, dashed, red!70!black] (13.1, 2.5) -- (13.1, -1.8);
    \draw[flow, line width=1.2pt] (12.7, 0) -- (13.7, 0) node[midway, above=2pt, font=\bfseries\scriptsize] {\textit{Unblock}};
    
    \filldraw[draw=red!70!black, fill=white, line width=1.2pt] (13.1, 1.5) circle (3.5pt);

    \node[draw=red!70!black, densely dotted, line width=1pt, rounded corners=4pt, inner sep=8pt, align=center, font=\small, text=red!70!black] (safetybox) at (17, 1.5) {
        \textbf{Phase I stops at $t=\tau$ (even)} \\[6pt]
        $B_\tau > 0 \quad \text{and} \quad \displaystyle \min_{i \in [k]} N_i \ge \frac{32 \sigma^2 L}{B_\tau^2}$
    };
    
    \draw[-, line width=1.2pt, draw=red!70!black] (13.25, 1.5) to[out=30, in=150] (safetybox.west);

    \node[ucb] (u1) at (14.5, 0) {UCB};
    \node[ucb] (u2) at (16.3, 0) {UCB};
    \node[ucb] (u3) at (18.1, 0) {UCB};
    
    \draw[flow] (15.15, 0) -- (15.65, 0);
    \draw[flow] (16.95, 0) -- (17.45, 0);
    \draw[flow, dashed] (18.75, 0) -- (19.2, 0);

    \foreach \x/\b in {2.0/1, 6.0/2, 11.0/{\ceil{\tau/2}}} {
        \draw[densely dotted, draw=gray!80, line width=0.8pt] (\x, 0.6) -- (\x, 1.5);
        \draw[line width=1pt, draw=black!80] (\x, 1.4) -- (\x, 1.6);
        \node[above, font=\small\bfseries, text=black!90] at (\x, 1.6) {$b=\b$};
    }

    \foreach \x/\t in {1.15/1, 2.85/2, 5.15/3, 6.85/4, 10.15/\tau-1, 11.85/\tau, 14.5/\tau+1, 16.3/\tau+2} {
        \draw[densely dotted, draw=gray!80, line width=0.8pt] (\x, -0.4) -- (\x, -1.2);
        \draw[line width=1pt, draw=black!80] (\x, -1.1) -- (\x, -1.3);
        \node[below, font=\small\bfseries, text=black!90] at (\x, -1.3) {$t=\t$};
    }

\end{tikzpicture}
}
\vskip -3mm
\caption{\textbf{Phase I} demonstrates \emph{harmonic anchored rank} exploration. \textbf{Sch} denotes the \textit{Scheduled Slot}, reserved for exploration while \textbf{Aux} denotes the \textit{Auxiliary Slot}.}
\label{fig:tiled_horizon}
\end{figure*}


The exact ordering of these slots in a block $b$ is decided by
a fair coin $\theta_b \sim \text{Ber}(1/2)$, drawn afresh when that block begins. The \textit{Scheduled Slot} comes first if $\theta_b = 1$, the \textit{Auxiliary Slot} otherwise. 

\noindent \emph{Scheduled Slot.}
Exploration in this slot follows a strict schedule where each block $b$ is assigned a fixed rank $r_b$ (equivalently, arm $\pi_{r_b}$) to explore. This \textit{schedule}, denoted by the sequence $(r_b)_{b \ge 1}$, is generated by iterating an integer $n = 1, 2, 3, \dots$ serially through the natural numbers and appending all divisors $d \le k$ of each $n$ in ascending order (see \cref{fig:harmonic-sequence}). 
Observe that this design exhibits a \emph{harmonic} property. Since any rank $r$ divides a $1/r$ fraction of consecutive integers, its exploration frequency is proportional to $1/r$. Also, the sequence, being generated entirely \textit{a priori}, remains independent of the reward history.

\begin{minipage}{\linewidth}
\centering
\resizebox{\textwidth}{!}{
$
    \def\vp{\vphantom{\xcancel{10}}}
    \begin{array}{r c c c c c c c c c c l}
        r_b : & 
        \underbracket{\; 1, \; \vp} & 
        \underbracket{\;1,\; 2, \vp} & 
        \underbracket{\;1,\; 3, \vp} & 
        \underbracket{\;1,\; 2,\; 4, \vp} & 
        \underbracket{\;1,\; 5, \vp} & 
        \underbracket{\;1,\; 2,\; 3,\; \textcolor{red}{\xcancel{\textcolor{black}{6,}}} \;\vp} & 
        \underbracket{\;1,\; \textcolor{red}{\xcancel{\textcolor{black}{7,}}} \;\vp} & 
        \underbracket{\;1,\; 2,\; 4,\; \textcolor{red}{\xcancel{\textcolor{black}{8,}}} \;\vp} & 
        \underbracket{\;1,\; 3,\; \textcolor{red}{\xcancel{\textcolor{black}{9,}}} \;\vp} & 
        \underbracket{\;1,\; 2,\; 5,\; \textcolor{red}{\xcancel{\textcolor{black}{10,}}} \;\vp} & 
        \dots \\[1.5ex]
        n : & 1 & 2 & 3 & 4 & 5 & 6 & 7 & 8 & 9 & 10 &\dots
    \end{array}
$
}
\vskip -2mm

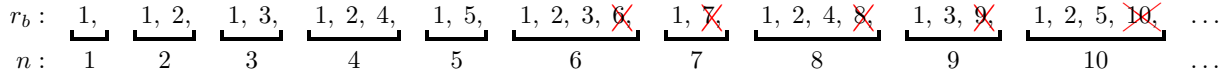
\captionof{figure}{(\textit{Harmonic schedule for $k = 5$ arms})  The sequence $r_b$ iteratively appends divisors of $n$, discarding any divisor $> k$. Each block $b$ maps to a rank $r_b$.}
\label{fig:harmonic-sequence}
\end{minipage}



\noindent \emph{Auxiliary Slot.}
This slot pulls the \emph{safest known arm}, if available, to defend the welfare floor. Let $c(n):=\sigma\sqrt{\frac{2\log(8kT/\delta)}{n}}$ denote the confidence radius after $n$ samples ($\delta \in (0,1)$ is the confidence level).
Writing $\hat{\mu}_i(t)$ and $N_i(t)$ for the empirical mean and number of pulls of arm $i$ up to (and including) round $t$, it computes the maximum lower confidence bound
\[ B_{t-1}:=\max\nolimits_{i\in[k]}\{\hat{\mu}_i(t-1)-c(N_i(t-1))\}, \]
where $t$ is its allotted round. If $B_{t-1}>0$, the maximizing arm is certified as the \textit{anchor} and is pulled. Otherwise, no \textit{anchor} has yet been certified, and the scheduled arm $\pi_{r_b}$ is pulled a second time.

Phase I continues blockwise, until an anchor has been found \textit{and} sufficient exploration has been performed across all arms. Formally, its termination occurs at stopping time $\tau$, defined as the first even round $t$ such that
\begin{equation}\label{eq:stop-condition}
    B_t > 0 \quad \text{and} \quad \min\nolimits_{i \in [k]} N_i(t) \ge 32\sigma^2 \log(8kT/\delta)/B_t^2.
\end{equation}
The first condition certifies the anchor, while the second adaptively requires every confidence radius to shrink below its certified margin. Together they bound the duration of Phase I, confining the inverse-power penalty to a finite transient before transitioning to UCB-based exploitation.

\textbf{Phase II: Exploitation.}
Upon satisfying \eqref{eq:stop-condition}, the block structure is discarded. For all subsequent rounds $t > \tau$, the algorithm employs the standard UCB index to pull arms, i.e., it selects the arm 
\[
I_t := \argmax\nolimits_{i \in [k]} \lbrace\muhat_i(t-1) + c(N_i(t-1))\rbrace.
\]

\textbf{Remarks on Algorithm Design.}
\textit{(i) Block boundaries.} The stopping condition is evaluated only at even rounds
(block boundaries), so the event \(\{\tau>t\}\) depends solely on the history prior to
block \(b\) and is independent of the intra-block coin flip \(\theta_b\). This preserves the
strict \(1/2\) marginal probabilities that the ex-ante reward bound (\Cref{lem:ex-ante-reward})
relies on; evaluating mid-block could truncate the auxiliary slot and break them.
\textit{(ii) Dynamic thresholding.} The anchor threshold \(B_{t-1}\) is read at the moment the
auxiliary slot executes, incorporating any scheduled observation already made in the same
block. The pseudo-code of \textsf{UCB-HARE} is presented in \Cref{alg:hare}.

\subsection{Upper Bound}\label{subsec:overview-upper}

We begin with a deterministic conversion that reduces the nonlinear \((-q)\)-mean regret to
a cumulative \textit{inverse penalty}. Define it as \(A_T:=\sum_{t=1}^T\bigl[(\mu_\star/m_t)^q-1\bigr]\), where \(m_t:=\E[\mu_{I_t}]\),
\(\mu_\star:=\max_{i\in[k]}\mu_i\).
\begin{lemma}[Inverse penalty conversion; informal, see \Cref{lem:inverse-conversion}]
\label{lem:informal-inverse-conversion}
\(R_{T,-q}\le\frac{\mu_\star}{q}\cdot\frac{A_T}{T}\).
\end{lemma}
\noindent To see this, factor \(\mu_\star\) out of the \((-q)\)-mean to get
\(\Pow_{-q}(m_1,\ldots,m_T)=\mu_\star\,(1+A_T/T)^{-1/q}\); the bound then follows from the
elementary inequality \(1-(1+u)^{-1/q}\le u/q\) for \(u\ge0\). \qed

Now we show that \textsf{UCB-HARE} keeps the ex-ante reward bounded away from zero quickly
enough during preparation, and then behaves like standard UCB after certification. Both
parts rest on the standard high-probability event on which the algorithm's confidence
intervals are valid. Let \(L:=\log(8kT/\delta)\) and \(c(n):=\sigma\sqrt{2L/n}\), and define
\(\calE:=\{\,\forall i\in[k],\ \forall n\le T:\ |\muhat_{i,n}-\mu_i|\le c(n)\,\}\). Applying
sub-Gaussian concentration to each of the \(kT\) arm--count pairs and a union bound gives
\(\Pr(\calE)\ge 1-\delta/4\). On \(\calE\) every lower confidence bound is valid,
\(L_i(t):=\muhat_i(t)-c(N_i(t))\le\mu_i\), so whenever \(B_t:=\max_i L_i(t)>0\), any maximizing
arm \(a\) satisfies \(\mu_a\ge L_a(t)=B_t>0\)  - certifying an \emph{anchor} with strictly
positive true mean.

The next lemma explains why Phase I cannot last too long. Define
\(n_\star:=\max\{\lceil 128\sigma^2L/\mu_\star^2\rceil,\,1\}\), \(S:=\;\lceil H_k n_\star\rceil\), and
\(T_0:=18kS\).
\begin{lemma}[Finite Preparation Window; informal, see \Cref{lem:certification,lem:prep-ends}]
\label{lem:informal-transition}
On \(\calE\), an anchor exists once the optimal arm has received \(n_\star\) samples, and
the harmonic schedule forces Phase I to terminate by round \(T_0=18kH_kn_\star\).
\end{lemma}
\noindent To see this, suppose the optimal arm \(i_\star\) has \(n_\star\) samples. Its
confidence radius is then \(c(N_{i_\star})\le\sigma\sqrt{2L/(128\sigma^2L/\mu_\star^2)}=\mu_\star/8\). So on \(\calE\),
\[
  L_{i_\star}(t)=\muhat_{i_\star}(t)-c(N_{i_\star}(t))\ge\mu_\star-2c(N_{i_\star}(t))\ge\tfrac{3}{4}\mu_\star,
\]
giving \(B_t>0\) and hence an anchor. For the duration, the prefix-balance property of the
harmonic schedule guarantees that by block \(8kS+1\) every rank (and thus every arm) has
at least \(n_\star\) samples. At that point the optimal arm certifies \(B_t\ge 3\mu_\star/4\),
and the stopping threshold satisfies
\(32\sigma^2L/B_t^2\le 32\sigma^2L/(3\mu_\star/4)^2=\tfrac{512}{9}\sigma^2L/\mu_\star^2\le n_\star\).
So the stopping condition holds at this block boundary, and since each block spans two
rounds, Phase I ends by \(T_0=18kS\). \qed

We split the cumulative \textit{inverse penalty} at this deterministic time,
\[
  A_T\le
  \underbrace{\sum\nolimits_{t=1}^{T\wedge T_0}\Bigl[(\mu_\star/m_t)^q-1\Bigr]}_{A_T^{\mathrm{prep}}}
  +\underbrace{\sum\nolimits_{t=T_0+1}^{T}\Bigl[(\mu_\star/m_t)^q-1\Bigr]}_{A_T^{\mathrm{UCB}}},
\]
with \(A_T^{\mathrm{UCB}}:=0\) if \(T\le T_0\). This is to avoid conditioning on the random
stopping time \(\tau\). We first control the preparation term. For block \(b\), define the
expanding frontier \(a_b:=\min\{k,\lfloor(b-1)/8S\rfloor\}\): by the start of block \(b\),
every arm whose random rank is at most \(a_b\) has accumulated enough samples to certify an
anchor, were it the optimal arm.
\begin{claim}[Growing reward floor; informal, see \Cref{lem:ex-ante-reward}]
\label{claim:informal-ex-ante}
There is a universal constant \(c_0>0\) (one may take \(c_0=1/4\)) such that for every round
\(t\), with \(b=\lceil t/2\rceil\), \(\ m_t\ge c_0\mu_\star\max\{1/k,\ a_b/k\}\).
\end{claim}
\noindent To see this, let \(R_\star\) be the rank of the optimal arm under the initial
permutation, so \(R_\star\) is uniform on \([k]\). The ex-ante reward has two sources. The
\emph{scheduled slot} is the current round with probability \(1/2\) (the within-block coin
flip), and the scheduled arm is optimal with probability \(1/k\), giving a baseline
contribution of order \(\mu_\star/k\). The \emph{auxiliary slot} contributes when
\(R_\star\le a_b\): the optimal arm then has \(n_\star\) samples by the start of block \(b\),
which on \(\calE\) certifies an anchor of mean at least \(3\mu_\star/4\). Since
\(\Pr(R_\star\le a_b)=a_b/k\), the round is the auxiliary slot with probability \(1/2\), and
the block-boundary stopping rule makes the continuation of Phase I independent of that coin
flip, this yields an auxiliary contribution of order \(\mu_\star a_b/k\). Finally, if Phase I
has already ended before round \(t\), the Phase-II rule is at least as safe (shown below in \cref{lem:informal-phase2-safe}), so
in all cases \(m_t\ge c_0\mu_\star\max\{1/k,\ a_b/k\}\). \qed

The reward floor yields a bound on the preparation penalty.
\begin{claim}[Preparation inverse penalty; informal, see \Cref{lem:prep-penalty}]
\label{claim:informal-prep-penalty}
There exists a constant \(C_q>0\), depending only on \(q\), such that
\[
  A_T^{\mathrm{prep}}=\sum\nolimits_{t=1}^{T\wedge T_0}\Bigl[(\mu_\star/m_t)^q-1\Bigr]
  \le C_q\,S\,K_q(k),
  \, \text{where}\,\, K_q(k):=\sum\nolimits_{r=1}^k\Bigl(\tfrac{k}{r}\Bigr)^q.
\]
\end{claim}
\noindent To see this, note the window \(t\le T\wedge T_0\) spans at most \(9kS\) blocks
(\(T_0=18kS\), two rounds per block), and by \Cref{claim:informal-ex-ante} every round in
block \(b\) has \((\mu_\star/m_t)^q\le C_q\max\{1/k,\,a_b/k\}^{-q}\). We partition the blocks
by the frontier value \(a_b\): the first \(8S\) blocks have \(a_b=0\) and contribute
\(O_q(Sk^q)\); each value \(a_b=r\) for \(r=1,\ldots,k-1\) persists for \(O(S)\) blocks,
contributing \(O_q\bigl(S(k/r)^q\bigr)\); and the saturated blocks (\(a_b=k\)) add \(O_q(S)\).
Summing up, we get $A_T^{\mathrm{prep}}\le C_q S\sum_{r=1}^k\Bigl(\tfrac{k}{r}\Bigr)^q=C_q S\,K_q(k)$. \qed

It remains to control the rounds after \(T_0\). The stopping rule ensures that, once Phase
II begins, every arm has a confidence radius small relative to \(\mu_\star\).
\begin{lemma}[Phase-II safety; informal, see \Cref{lem:phase2-safe}]
\label{lem:informal-phase2-safe}
On \(\calE\), if Phase I stops at round \(\tau\), then every arm selected by the UCB rule in
Phase II has a true mean of at least \(\mu_\star/2\).
\end{lemma}
\noindent To see this, recall that at the stopping time \(B_\tau>0\) and
\(\min_i N_i(\tau)\ge 32\sigma^2L/B_\tau^2\). On \(\calE\) no lower confidence bound exceeds
its true mean, so \(B_\tau\le\mu_\star\); hence every arm has at least
\(32\sigma^2L/\mu_\star^2\) samples and confidence radius at most \(\mu_\star/4\). Now let
\(i\) be any arm selected by UCB in Phase II. Its index dominates the optimal arms',
\(\muhat_i(t)+c(N_i(t))\ge\muhat_{i_\star}(t)+c(N_{i_\star}(t))\ge\mu_\star\), so
\[
  \mu_i\ge\muhat_i(t)-c(N_i(t))=\bigl(\muhat_i(t)+c(N_i(t))\bigr)-2c(N_i(t))
  \ge\mu_\star-\tfrac{\mu_\star}{2}=\tfrac{\mu_\star}{2}.
\] 

This safety property lets the Phase-II penalty be controlled by UCB pseudo-regret.
\begin{claim}[Phase-II inverse penalty; informal, see \Cref{lem:phase2-cost}]
\label{claim:informal-phase2-cost}
There exists a constant \(C_q>0\), depending only on \(q\), such that
\[
A_T^{\mathrm{UCB}}=\sum\nolimits_{t=T_0+1}^{T}\Bigl[(\mu_\star/m_t)^q-1\Bigr]
  \le C_q\Bigl(\tfrac{\sigma}{\mu_\star}\sqrt{kTL}+T\delta\Bigr).
\]
\end{claim}
\noindent To see this, assume \(T>T_0\) (else the sum is empty). By
\Cref{lem:informal-transition}, all rounds \(t>T_0\) lie in Phase II, and on \(\calE\),
\Cref{lem:informal-phase2-safe} gives selected mean at least \(\mu_\star/2\); since
\(\Pr(\calE)\ge 1-\delta/4\), this yields the unconditional floor \(m_t\ge\mu_\star/4\). On
\(m_t\in[\mu_\star/4,\mu_\star]\) the inverse-power map is locally Lipschitz, so
\((\mu_\star/m_t)^q-1\le C_q(\mu_\star-m_t)/\mu_\star\). The numerator \(\mu_\star-m_t\) is
the expected pseudo-regret at round \(t\); on \(\calE\) the standard UCB argument bounds its
cumulative sum over Phase II by \(O(\sigma\sqrt{kTL})\), while \(\calE^c\) contributes at
most \(O(\mu_\star T\delta)\). Dividing by \(\mu_\star\) gives the displayed bound. \qed

Combining \Cref{claim:informal-prep-penalty,claim:informal-phase2-cost} gives $A_T\le C_q\Bigl(S\,K_q(k)+\tfrac{\sigma}{\mu_\star}\sqrt{kTL}+T\delta\Bigr)$, where $S= \lceil H_kn_\star \rceil$.
Since \(n_\star\le 1+128\sigma^2L/\mu_\star^2\), the preparation term obeys
\(S K_q(k)\le H_kK_q(k)\,(1+C\sigma^2L/\mu_\star^2)\), and applying
\Cref{lem:informal-inverse-conversion} yields the $(-q)$-mean regret
\[
  R_{T,-q}\le C_q\left[
    \frac{\mu_\star H_kK_q(k)}{T}
    +\frac{\sigma^2L\,H_kK_q(k)}{\mu_\star T}
    +\sigma\sqrt{\frac{kL}{T}}
    +\mu_\star\delta\right].
\]
Because \(H_kK_q(k)\ge k\), the third term satisfies \(\sigma\sqrt{kL/T}\le\mu_0\), where
\(\mu_0:=\sigma\sqrt{H_kK_q(k)L/T}\). For the second term, note that it equals exactly
\(\mu_0^2/\mu_\star\), which is large only when \(\mu_\star\) is small; there, however, the
trivial bound \(R_{T,-q}\le\mu_\star\) takes over. Combining the two, we get
$\min\Bigl\{\mu_\star,\ \tfrac{\mu_0^2}{\mu_\star}\Bigr\}\le\mu_0$,
so the second and third terms together contribute at most \(O(\mu_0)\). The first and
fourth terms carry through unchanged. Setting $\delta=1/T$ gives
\[
  R_{T,-q}\le C_q\left[\sigma\sqrt{\frac{H_kK_q(k)L}{T}}+\frac{\mu_\star H_kK_q(k)}{T}\right].
\]


\subsection*{Paper Organization}

In \Cref{sec:prelims}, we collect the main technical ingredients used throughout
the paper. We begin with basic properties of $(-q)$-mean regret, including the
reductions that convert the nonlinear welfare objective into cumulative
inverse-penalty bounds. We then review the information-theoretic tools needed
for the minimax lower bound, and record the deterministic consequences of the
harmonic schedule that drives the exploration phase of \hare. With these
preliminaries established, \Cref{sec:lower-bounds} proves the lower bound for
the strongly fair regime, showing how the inverse-power objective amplifies the
cost of slow identification of the optimal arm. \Cref{sec:algo} presents the
pseudocode of \hare\ together with its main guarantee. The matching upper bound
for \hare\ is proved in \Cref{sec:upper-proof}, where we combine the reward
floor created during preparation with the safety of the subsequent UCB phase.
Finally, \Cref{sec:expt} presents our experimental evaluation and compares
\hare\ against uniform-exploration-based baselines.
\clearpage

\addtocontents{toc}{\protect\setcounter{tocdepth}{2}}

\section{Preliminaries}\label{sec:prelims}
We operate in the standard $k$-armed stochastic bandit setting over a horizon $T$. Let $m_t := \E[\mu_{I_t}]$ denote the ex-ante expected reward at round $t$, and let $\mustar := \max_{i \in [k]} \mu_i$. In this section, we establish some prerequisites for our analysis.\Cref{subsec:regret} introduces the welfare objective and the specific reductions we use to bound it. \Cref{subsec:infothy} reviews the information-theoretic tools necessary for our minimax lower bounds. Finally, \Cref{subsec:harmony} defines the deterministic schedule driving our algorithmic upper bound.

\subsection{The $p$-Mean Regret}\label{subsec:regret}

Our objective is to evaluate collective welfare using the generalized $p$-mean functional \cite{krishna2025pmean, moulin2004fair} in the negative (fair) regime, which penalizes suboptimal assignments to early rounds. Parameterizing this regime by $q = -p > 0$, we define the $(-q)$-mean welfare of the ex-ante reward sequence as:
\[
    \Pow_{-q}(m_1, \dots, m_T) := \left(\frac{1}{T} \sum_{t=1}^T m_t^{-q} \right)^{-1/q}.
\]
The $(-q)$-mean regret of an algorithm $\Alg$ on instance $\nu$ is its welfare shortfall relative to the optimal fixed-arm policy:
\[
    R_{T,-q}(\Alg, \nu) := \mustar - \Pow_{-q}(m_1, \ldots, m_T).
\]

We begin with three properties. First, \Cref{lem:power-monotone} establishes monotonicity of the welfare function, which we use in our lower bound proof in \Cref{sec:lower-bounds}. 

\begin{lemma}[Power mean monotonicity]\label{lem:power-monotone}
Let $x_1,\ldots,x_T\ge0$. If $p\le r$, then
\[
    \Pow_p(x_1,\ldots,x_T)\le \Pow_r(x_1,\ldots,x_T),
\]
with the usual continuous interpretation at $p=0$.
\end{lemma}

\begin{proof}
For positive coordinates, this is the classical generalized mean inequality derived via Jensen's inequality. For coordinates equal to zero, the result follows by continuity, under the convention that negative-order means evaluate to zero if any coordinate is zero.
\end{proof}


The remaining two results aid \Cref{sec:upper-proof}, where we establish upper bounds for \hare. \Cref{lem:inverse-conversion} first isolates the outer non-linearity of the regret into a cumulative \textit{inverse penalty}. \Cref{lem:lipschitz-inverse} then handles the remaining inner non-linearity, reducing this penalty to the standard linear pseudo-regret by exploiting its local Lipschitz continuity when rewards are bounded away from zero.

\begin{lemma}[First Linearization: Regret to Penalty]\label{lem:inverse-conversion}
Let $q>0$, $\mustar>0$, and recall $m_t\in(0,\mustar]$ for all $t$. Define
$
    A_T:=\sum_{t=1}^T\left[\left(\frac{\mustar}{m_t}\right)^q-1\right].
$
Then
\[
    R_{T,-q}(\Alg, \nu)
    \le
    \frac{\mustar}{q}\frac{A_T}{T}.
\]
\end{lemma}
\begin{proof}
Factoring out $\mustar$ and substituting $A_T$,
\begin{align*}
    R_{T,-q}(\Alg, \nu) 
    &= \mustar - \Pow_{-q}(m_1, \ldots, m_T)\\
    &= \mustar \left[ 1 - \left( \frac{1}{T} \sum_{t=1}^T \left(\frac{\mustar}{m_t}\right)^q \right)^{-1/q} \right] \\
    &= \mustar \left[ 1 - \left( 1 + \frac{A_T}{T} \right)^{-1/q} \right] \\
    &\le \mustar \left[ 1 - \left( 1 - \frac{A_T}{qT} \right) \right]
    = \frac{\mustar}{q} \frac{A_T}{T},
\end{align*}
where the inequality uses $(1+u)^{-1/q} \ge 1 - u/q$ for $u \ge 0$, which holds because $u \mapsto (1+u)^{-1/q}$ is convex when $q > 0$, and $1-u/q$ is its tangent line at $u=0$.
\end{proof}


\begin{lemma}[Second Linearization: Penalty to Pseudo-Regret]\label{lem:lipschitz-inverse}
For every $q>0$, there exists $C_q>0$ such that for all $x\in[1/4,1]$,
\[
    x^{-q}-1\le C_q(1-x).
\]
Consequently, if $y\in[\mustar/4,\mustar]$, then
\[
    \left(\frac{\mustar}{y}\right)^q-1
    \le
    C_q\frac{\mustar-y}{\mustar}.
\]
\end{lemma}

\begin{proof}
Let $f(x) = x^{-q}$. The bound holds for $x=1$. For $x \in [1/4, 1)$, the Mean Value Theorem guarantees some $c \in (x, 1)$ such that $f(1) - f(x) = f'(c)(1 - x)$. As $f(1)=1$, we get
\begin{align*}
    x^{-q} - 1 
    &= q c^{-q-1}(1 - x) \\
    &< q 4^{q+1}(1 - x),
\end{align*}
where the inequality follows because $c \mapsto c^{-q-1}$ is strictly decreasing and $c > x \ge 1/4$. Setting $C_q := q 4^{q+1}$ establishes the primary inequality. 

The consequence follows by substituting $x = y/\mustar$ for any $y \in [\mustar/4, \mustar]$.
\end{proof}

\subsection{Useful Results From Information Theory}\label{subsec:infothy}

To establish the fundamental limits of \textit{fair} exploration, we rely on standard information-theoretic quantities. For two probability distributions $P$ and $Q$ defined on a common measurable space, denote the  Kullback-Leibler (KL) divergence by $\KL(P \,\|\, Q)$. For Bernoulli distributions with parameters $p, q \in (0, 1)$, we denote the binary relative entropy by 
\[ 
    \kl(p \,\|\, q) := p \log\left(\frac{p}{q}\right) + (1-p) \log\left(\frac{1-p}{1-q}\right). 
\]
For our lower bound analysis, we require two structural properties of mutual information. The first allows us to upper-bound the conditional mutual information with respect to an arbitrary reference distribution. This is crucial for isolating the information gain exclusively to rounds where the optimal arm is pulled.

\vspace{0.5em}
\begin{lemma}\label{lem:bound-mi}
Let $X$ and $Y$ be random variables, and let $P_{Y|X}$ denote the conditional distribution of $Y$ given $X$. For any fixed reference distribution $Q$ over the support of $Y$, the mutual information satisfies:
\[
    I(X; Y) \le \E_X \big[ \KL(P_{Y \mid X} \,\|\, Q) \big].
\]
\end{lemma}

\begin{proof}
By definition, $I(X; Y) = \E_X [\KL(P_{Y \mid X} \,\|\, P_Y)]$, where $P_Y$ is the marginal distribution of $Y$. For any reference distribution $Q$, we can expand the divergence as $\KL(P_{Y \mid X} \,\|\, Q) = \KL(P_{Y \mid X} \,\|\, P_Y) + \E_{Y \mid X}[\log(P_Y(Y)/Q(Y))]$. Taking the expectation over $X$ yields $\E_X[\KL(P_{Y \mid X} \,\|\, Q)] = I(X; Y) + \KL(P_Y \,\|\, Q)$. Since KL divergence is non-negative, the bound follows.
\end{proof}

The second property establishes a lower bound on the mutual information required to identify a uniformly drawn hypothesis. We express this via the function $d_k : [1/k, 1] \to \mathbb{R}_+$, defined as
\[
    d_k(s) := s \log(ks) + (1-s) \log\left(\frac{k(1-s)}{k-1}\right) = \kl\left(s \,\middle\|\, \frac{1}{k}\right).
\]
Note that $d_k(s)$ is strictly increasing on $[1/k, 1]$.

\begin{lemma}[Mutual Information Fano Converse]\label{lem:fano-converse}
Let $J$ be a random variable drawn uniformly from $[k]$. Let $\hat{J}$ be any estimator of $J$, and define the success probability as $s := \Pbb(\hat{J} = J)$. Then the mutual information satisfies:
\[
    I(J; \hat{J}) \ge d_k(s).
\]
\end{lemma}

\begin{proof}
Let $E = \1\{J \neq \hat{J}\}$ denote the binary error indicator, where $\Pbb(E = 0) = s$. We bound the conditional entropy $H(J \mid \hat{J})$ by expanding the joint conditional entropy $H(E, J \mid \hat{J})$ via the chain rule in two ways.

First, $H(E, J \mid \hat{J}) = H(J \mid \hat{J}) + H(E \mid J, \hat{J})$. As $E$ is deterministically computed from $J$ and $\hat{J}$, the second term vanishes, yielding $H(E, J \mid \hat{J}) = H(J \mid \hat{J})$.

Second, expanding with respect to $E$ yields $H(E, J \mid \hat{J}) = H(E \mid \hat{J}) + H(J \mid E, \hat{J})$. We bound the components on the right-hand side individually. Conditioning reduces entropy, so 
$$H(E \mid \hat{J}) \le H(E) = H_b(s) := -s \log s - (1-s) \log(1-s),$$ 
where $H_b$ is the binary entropy function.

For the remaining term, we condition on the realization of $E$:
\begin{enumerate}
    \item If $E=0$, $J$ is exactly $\hat{J}$, yielding zero conditional entropy.
    \item If $E=1$, $J$ must take one of the remaining $k-1$ values, which has entropy at most $\log(k-1)$.
\end{enumerate}
Averaging over the indicator $E$ yields $H(J \mid E, \hat{J}) \le s(0) + (1-s)\log(k-1)$. Equating the two expansions bounds the conditional entropy, recovering Fano's inequality:
\[
    H(J \mid \hat{J}) \le H_b(s) + (1-s)\log(k-1).
\]

By definition, the mutual information is $I(J; \hat{J}) = H(J) - H(J \mid \hat{J})$. Since $J$ is uniformly distributed over $[k]$, its prior entropy is $H(J) = \log k$. We substitute our bounds and group the terms as follows:
\begin{align*}
    I(J; \hat{J}) 
    &\ge \log k - \Big[ H_b(s) + (1-s)\log(k-1) \Big] \\
    &= \big(s \log k + (1-s) \log k\big) + s \log s + (1-s) \log(1-s) - (1-s) \log(k-1)  \\
    &= s(\log k + \log s) + (1-s) \big[ \log k + \log(1-s) - \log(k-1) \big] \\
    &= s \log(k s) + (1-s) \log\left(\frac{k(1-s)}{k-1}\right) \\
    &= d_k(s).
\end{align*}
\end{proof}

\subsection{Harmonic Schedule}\label{subsec:harmony}

To coordinate exploration, our algorithm relies on a deterministic, rank-based sequence which we call the \textit{harmonic schedule}. For a fixed number of arms $k$ and any integer $n \ge 1$, define the truncated divisor set $\mathcal{D}_n := \{d \in [k] : d \mid n\}$. 

The harmonic schedule, denoted $(r_b)_{b \ge 1}$, is constructed by concatenating the sets $\mathcal{D}_1, \mathcal{D}_2, \dots$ in ascending order of $n$. For example, setting $k=5$, the sequence discards any divisors strictly greater than $k$, yielding:

\vspace{0.5em}
\begin{minipage}{\linewidth}
\centering
\resizebox{\textwidth}{!}{
$
    \def\vp{\vphantom{\xcancel{10}}} 
    \begin{array}{r c c c c c c c c c c l}
        r_b : & 
        \underbracket{\; 1, \; \vp} & 
        \underbracket{\;1,\; 2, \vp} & 
        \underbracket{\;1,\; 3, \vp} & 
        \underbracket{\;1,\; 2,\; 4, \vp} & 
        \underbracket{\;1,\; 5, \vp} & 
        \underbracket{\;1,\; 2,\; 3,\; \textcolor{red}{\xcancel{\textcolor{black}{6,}}} \;\vp} & 
        \underbracket{\;1,\; \textcolor{red}{\xcancel{\textcolor{black}{7,}}} \;\vp} & 
        \underbracket{\;1,\; 2,\; 4,\; \textcolor{red}{\xcancel{\textcolor{black}{8,}}} \;\vp} & 
        \underbracket{\;1,\; 3,\; \textcolor{red}{\xcancel{\textcolor{black}{9,}}} \;\vp} & 
        \underbracket{\;1,\; 2,\; 5,\; \textcolor{red}{\xcancel{\textcolor{black}{10,}}} \;\vp} & 
        \dots \\[1.5ex]
        n : & 1 & 2 & 3 & 4 & 5 & 6 & 7 & 8 & 9 & 10 &\dots
    \end{array}
$
}
\end{minipage}
\vspace{1em}

The primary advantage of this schedule is the following lemma, which allows us to track the pull count of each rank.

\begin{lemma}[Prefix Balance Lemma]\label{lem:prefix-balance}
For any rank $r \in [k]$ and block $b \ge 0$, denote the number of scheduled pulls allocated to rank $r$ up to and including block $b$, by $N_r^{\rm sch}(b) := \#\{s \le b : r_s = r\}$. Then, 
\begin{equation*}
    N_r^{\rm sch}(b)
    \ge
    \frac{b}{rH_k} - 2,
\end{equation*}
where $H_k := \sum_{r=1}^k 1/r$.
\end{lemma}

\begin{proof}
Based on our harmonic schedule, let block $b$ be associated with integer $n_b$ (recall \cref{subsec:overview-algo}). Comparing index $b$ against that of the last block mapped to $n_b$, we have
\[
b \leq \sum_{n=1}^{n_b} |\mathcal{D}_n|
= \sum_{r=1}^k \left\lfloor \frac{n_b}{r} \right\rfloor
\leq n_b \sum_{r=1}^k \frac{1}{r}
= n_b H_k
\]
and hence $n_b \ge b / H_k$.

For any rank $r \in [k]$, its occurrences in the schedule up to block $b$ must at least include every occurrence within the divisor sets $\mathcal{D}_1, \dots, \mathcal{D}_{n_b - 1}$ strictly prior to $n_b$. Therefore,
\[
N_r^{\rm sch}(b) 
\geq \sum_{n=1}^{n_b - 1} \mathbbm{1}\{r \in \mathcal{D}_n\} 
= \sum_{n=1}^{n_b - 1} \mathbbm{1}\{r \mid n\} 
= \left\lfloor \frac{n_b - 1}{r} \right\rfloor
\geq \frac{n_b}{r} - 2
\geq \frac{b}{r H_k} - 2.
\]
\end{proof}

We also preemptively record here the asymptotic behavior of a certain \textit{harmonic penalty sum} induced by the schedule (used later in \Cref{sec:upper-proof}), which governs the final regret regime.

\begin{fact}[Asymptotics of the Harmonic Penalty]\label{fact:harmonic-penalty}
For any $q > 0$ and $k \ge 1$, define the harmonic penalty sum as $K_q(k) := \sum_{r=1}^k (k/r)^q$. Then:
\[
    K_q(k) = 
    \begin{cases}
        \Theta_q(k), & 0 < q < 1, \\
        \Theta(k \log k), & q = 1, \\
        \Theta_q(k^q), & q > 1.
    \end{cases}
\]
\end{fact}

\begin{proof}
We can factor the sum as $K_q(k) = k^q \sum_{j=1}^k j^{-q}$. The sum $S_q(k) := \sum_{j=1}^k j^{-q}$ can be bounded by comparing it to the integral of the monotonically decreasing function $f(x) = x^{-q}$ over the continuous interval $[1, k]$:
\[
    \int_1^{k+1} x^{-q} \, dx \le S_q(k) \le 1 + \int_1^k x^{-q} \, dx.
\]
Evaluating the integrals partitions the sequence into three standard regimes:
\begin{enumerate}
    \item \textbf{Sub-linear ($0 < q < 1$):} The integral evaluates to $\frac{x^{1-q}}{1-q}$, yielding $S_q(k) = \Theta_q(k^{1-q})$. Multiplying by the leading $k^q$ factor gives $k^q \cdot \Theta_q(k^{1-q}) = \Theta_q(k)$.
    \item \textbf{Logarithmic ($q = 1$):} The integral evaluates to $\log x$, bounding the harmonic number as $\log(k+1) \le S_1(k) \le 1 + \log k$, yielding $S_1(k) = \Theta(\log k)$. Multiplying by $k^1$ yields $\Theta(k \log k)$.
    \item \textbf{Convergent ($q > 1$):} The improper integral $\int_1^\infty x^{-q} \, dx = \frac{1}{q-1}$ is finite, implying $S_q(k)$ converges to a constant dependent on $q$. Multiplying by $k^q$ gives $\Theta_q(k^q)$.
\end{enumerate}
Substituting these bounds establishes the claim.
\end{proof}
\section{Our Lower Bound and Proof}\label{sec:lower-bounds}

In this section, we establish the information-theoretic limits of strongly fair exploration. We formally restate our main lower bound theorem below, which highlights that the minimax regret inflates from $\sqrt{k/T}$ to $\sqrt{k^q/T}$.

\begin{theorem}[Main Lower Bound, Restated]\label{thm:lower-main-restated}
For any fairness parameter $q > 0$, number of arms $k \ge 8$, and time horizon $T \ge 1$, the minimax $(-q)$-mean regret of any algorithm $\Alg$ is bounded as
\[
    \mregret_{T,-q}(k, \sigma) := \inf_\Alg \sup_{\nu \in \Esg{k}{\sigma}} R_{T,-q}(\Alg, \nu) \ge c_q \sigma \sqrt{\frac{k^q}{T}},
\]
where $c_q > 0$ is a constant depending only on $q$. 
\end{theorem}

We first establish the classical baseline for $p \le 1$.

\begin{proposition}[Average Regret Baseline]\label{prop:avg-lb-main}
For any power parameter $p \le 1$, number of arms $k \ge 2$, and horizon $T \ge 1$:
\[
    \mregret_{T,p}(k, \sigma) \ge C \sigma \sqrt{\frac{k}{T}},
\]
where $C > 0$ is a universal constant.
\end{proposition}

\begin{proof}
For $p \le 1$, power mean monotonicity (\Cref{lem:power-monotone}) bounds the $p$-mean by the arithmetic mean. Thus, the generalized regret reduces to the usual average regret:
\[
    R_{T,p}(\Alg, \nu) \ge \mustar - \frac{1}{T}\sum_{t=1}^T m_t^\nu.
\]
To lower bound the minimax over the sub-gaussian class $\Esg{k}{\sigma}$, we restrict the supremum to the subset of \textit{Gaussian} instances with variance $\sigma^2$. As this Gaussian family is inherently $\sigma$-sub-Gaussian, the standard information-theoretic reduction over this hard subset \cite{bubeck2012regret} directly establishes the $\Omega(\sigma\sqrt{k/T})$ baseline.
\end{proof}

For our main proof, we focus on $k$ difficult ``one-good-arm'' instances and average the regret over them. Fix $q>0$, $k\ge8$ and $\mu>0$. For each $j\in[k]$, define the underlying instance $\nu^{(j)}$ by:
\[
    \nu_i^{(j)} =
    \begin{cases}
    \Gauss(\mu,\sigma^2), & i=j,\\
    \Gauss(0,\sigma^2), & i\neq j.
    \end{cases}
\]
Let $\Pbb_j$ denote the probability measure under instance $\nu^{(j)}$. Let $J \sim \Unif([k])$ 
index a uniformly random instance, and define $\Pbb$ as the law averaged over this random choice,
\[
\Pbb(\cdot) := \frac{1}{k}\sum_{j=1}^k \Pbb_j(\cdot)
\]
under which $J$ is first drawn and then instance $\nu^{(J)}$ is executed.
We define the algorithm's instance-specific success probability $a_{j,t}$ and average success probability $s_t$ at round $t$ as:
\[
    a_{j,t} := \Pbb_j(I_t=j),
    \qquad
    s_t := \Pbb(I_t=J) = \frac{1}{k}\sum_{j=1}^k a_{j,t}.
\]

\begin{lemma}[A Fixed-Instance Reduction]\label{lem:bayes-reduction}
There exists an instance $j\in[k]$ such that the time-averaged penalty is bounded below by the ensemble averaged penalty:
\[
    \frac{1}{T}\sum_{t=1}^T a_{j,t}^{-q} \ge \frac{1}{T}\sum_{t=1}^T s_t^{-q},
\]
where we define $0^{-q} := +\infty$.
\end{lemma}

\begin{proof}
Because the map $x \mapsto x^{-q}$ is strictly convex on $(0,\infty)$, Jensen's inequality implies that for any round $t$:
\[
    \frac{1}{k}\sum_{j=1}^k a_{j,t}^{-q} \ge \left(\frac{1}{k}\sum_{j=1}^k a_{j,t}\right)^{-q} = s_t^{-q}.
\]
Averaging both sides over all rounds $t \in [T]$ yields:
\[
    \frac{1}{k}\sum_{j=1}^k \left( \frac{1}{T}\sum_{t=1}^T a_{j,t}^{-q} \right) \ge \frac{1}{T}\sum_{t=1}^T s_t^{-q}.
\]
By the pigeonhole principle, the maximum over the $k$ instances must be at least as large as the ensemble average, meaning there must exist at least one index $j \in [k]$ satisfying the claim.
\end{proof}

To bound the success probability $s_t$, we track the mutual information collected by any algorithm on $J$ prior to round $t$, denoted as $\mathcal{I}_t := I(J; \mathcal{H}_{t-1})$.
Recall from our preliminaries (\Cref{subsec:infothy}) that the function $d_k(s) := \kl(s \,\|\, 1/k)$ represents the minimum information required to guess the optimal arm with success probability $s$. 
\begin{lemma}[Information Lower Bound on Success]\label{lem:success-info}
Let $r_t := \max\{s_t, 1/k\}$. Then 
\[
d_k(r_t) \le \mathcal{I}_t.
\]
\end{lemma}

\begin{proof}
If $s_t < 1/k$, then $r_t = 1/k$. Because $d_k(1/k) = 0$ and mutual information is strictly non-negative, the bound $d_k(r_t) \le \mathcal{I}_t$ holds trivially.

If $s_t \ge 1/k$, then $r_t = s_t$. We can view the algorithm's action $I_t$ as an estimator of the optimal arm $J$. Because $I_t$ is generated using only the observed history $\mathcal{H}_{t-1}$ and independent randomization, the Data Processing Inequality implies $I(J; I_t) \le I(J; \mathcal{H}_{t-1}) = \mathcal{I}_t$. Finally, \Cref{lem:fano-converse} bounds the information from below as
\[
    \mathcal{I}_t \ge I(J; I_t) \ge d_k(s_t) = d_k(r_t).
\]
\end{proof}

\begin{lemma}[One-Step Information Gain]\label{lem:info-gain}
Let $\lambda := \KL(\Gauss(\mu, \sigma^2) \,\|\, \Gauss(0, \sigma^2)) = \frac{\mu^2}{2\sigma^2}$. Then for every round $t$, the information gain satisfies:
\[
    \mathcal{I}_{t+1} - \mathcal{I}_t \le \lambda s_t.
\]
\end{lemma}

\begin{proof}
By the chain rule of mutual information and the conditional independence of an algorithm's internal randomization, the step-wise information gain is exactly $\mathcal{I}_{t+1} - \mathcal{I}_t = I(J; Y_t \mid \mathcal{H}_{t-1}, I_t)$. 

We invoke \Cref{lem:bound-mi} using the suboptimal arm distribution $Q = \Gauss(0, \sigma^2)$ as the reference distribution. As the true reward distribution matches $Q$ exactly whenever $J \neq I_t$, the conditional KL divergence vanishes entirely on suboptimal pulls, isolating the penalty to just the optimal pulls. Hence,
\begin{align*}
    I(J; Y_t \mid \mathcal{H}_{t-1}, I_t) 
    &\le \E_{J, I_t, \mathcal{H}_{t-1}} \Big[ \KL\big(P_{Y_t \mid J, I_t, \mathcal{H}_{t-1}} \,\|\, Q\big) \Big] 
        && \text{(conditional reference bound)} \\
    &= \E_{J, I_t} \Big[ \1\{J = I_t\} \KL(\Gauss(\mu, \sigma^2) \,\|\, \Gauss(0, \sigma^2)) \Big] 
        && \text{(marginalizing out $\mathcal{H}_{t-1}$)} \\
    &= \lambda \, \E_{J, I_t} \big[ \1\{J = I_t\} \big] = \lambda s_t.
        && \text{(definition of success probability)}
\end{align*}
\end{proof}

\begin{lemma}[Main Information-Cost Inequality]\label{lem:det-cost}
Let $q > 0$ and $k \ge 8$. Suppose a sequence $r_t \in [1/k, 1]$ and a non-decreasing sequence $\mathcal{I}_t$ satisfy $\mathcal{I}_1 = 0$, $d_k(r_t) \le \mathcal{I}_t$, and $\mathcal{I}_{t+1} - \mathcal{I}_t \le \lambda r_t$ for all $t \in [T]$. Then there exists a constant $c_q > 0$, depending only on $q$, such that:
\[
    \sum_{t=1}^T (r_t^{-q} - 1) \ge c_q \min\left\{T, \frac{k^q}{\lambda}\right\}.
\]
\end{lemma}

\begin{proof}
Define the continuous extension $D(s) := d_k(s)$, which is strictly increasing on the interval $[1/k, 1]$. We split the analysis into two cases based on the fact that in the first round, the information crosses a constant threshold. Let $\tau := \inf\{t \in [T] : \mathcal{I}_t \ge D(1/2)\}$, setting $\tau = +\infty$ if this never occurs.

\textbf{Case 1 ($\tau = +\infty$):} For all rounds $t$, we have $D(r_t) \le \mathcal{I}_t < D(1/2)$. Because $D$ is strictly increasing, this implies $r_t < 1/2$. Therefore, $r_t^{-q} - 1 \ge 2^q - 1$. Summing this constant over all $T$ rounds immediately satisfies the bound.

\textbf{Case 2 ($\tau < +\infty$):} We focus strictly on the rounds before the threshold ($t < \tau$). Define the truncated step-wise increments $\delta_t := \min\{\mathcal{I}_{t+1}, D(1/2)\} - \mathcal{I}_t \ge 0$. By our premise, $\delta_t \le \mathcal{I}_{t+1} - \mathcal{I}_t \le \lambda r_t$. Since $r_t < 1/2$ prior to $\tau$, we can bound the penalty term as $r_t^{-q} - 1 \ge (1 - 2^{-q})r_t^{-q}$. 

Because $D$ is increasing and $D(r_t) \le \mathcal{I}_t$, we know $r_t \le D^{-1}(\mathcal{I}_t)$. We introduce the non-increasing function $f(u) := (D^{-1}(u))^{-q-1}$ to smoothly lower bound the discrete sum via a left-endpoint Riemann integral:
\begin{align*}
    \sum_{t < \tau} r_t^{-q} 
    &\ge \frac{1}{\lambda} \sum_{t < \tau} r_t^{-q-1} \delta_t 
        && \text{(since $\lambda r_t \ge \delta_t$)} \\
    &\ge \frac{1}{\lambda} \sum_{t < \tau} f(\mathcal{I}_t) \delta_t 
        && \text{(applying $r_t \le D^{-1}(\mathcal{I}_t)$)} \\
    &\ge \frac{1}{\lambda} \int_0^{D(1/2)} f(u) \, du. 
        && \text{(left-endpoint \textit{upper-sum} inequality)}
\end{align*}
We evaluate this integral using the change of variables $u = D(s)$, which gives $du = D'(s)ds$:
\begin{align*}
    \int_0^{D(1/2)} f(u) \, du 
    &= \int_{1/k}^{1/2} s^{-q-1} D'(s) \, ds \\
    &\ge \int_{2/k}^{4/k} s^{-q-1} \log\left(\frac{sk}{2}\right) \, ds 
        && \text{(since $D'(s) \ge \log(sk/2)$ for $k \ge 8$)} \\
    &= \frac{k^q}{2^q} \int_1^2 u^{-q-1} \log u \, du \ge C_q k^q. 
        && \text{(substituting $u = sk/2$)}
\end{align*}
Applying the strictly positive integral constant $C_q$ yields $\sum_{t < \tau} (r_t^{-q} - 1) \ge c_q k^q / \lambda$. Decreasing $c_q$ if necessary to satisfy both cases simultaneously completes the proof.
\end{proof}

\begin{proof}[\textbf{Proof of \Cref{thm:lower-main-restated}}]
Fix an arbitrary algorithm $\Alg$. We instantiate the $k$ underlying instances with $\mu := \sigma\sqrt{k^q/T}$, which sets the KL-divergence multiplier to $\lambda = \mu^2 / (2\sigma^2) = k^q / (2T)$.

Notice that replacing any $s_t < 1/k$ with $r_t = 1/k$ strictly decreases the inverse-power mean. As such, we apply \Cref{lem:info-gain}, \Cref{lem:success-info}, and \Cref{lem:det-cost} to the unconditional success probabilities, and get
\[
    \frac{1}{T}\sum_{t=1}^T s_t^{-q} 
    \ge 1 + \frac{c_q}{T} \min\left\{T, \frac{k^q}{\lambda}\right\} 
    = 1 + c_q \min\{1, 2\} = 1 + c_q.
\]
By \Cref{lem:bayes-reduction}, there exists at least one specific instance $j \in [k]$ whose time-averaged penalty $\mathcal{A}_j := \frac{1}{T}\sum_{t=1}^T a_{j,t}^{-q}$ satisfies $\mathcal{A}_j \ge 1 + c_q$.

Under this specific instance $\nu^{(j)}$, the optimal arm is $j$ with mean $\mustar = \mu$, and the expected reward of the algorithm at round $t$ is exactly $m_t^{(j)} = \mu a_{j,t}$. If $\mathcal{A}_j = +\infty$, the $(-q)$-mean welfare is strictly zero and the regret is trivially $\mu$. Otherwise, we compute the algorithm's expected regret directly from the definition of the $(-q)$-mean:
\begin{align*}
    R_{T,-q}(\Alg, \nu^{(j)}) 
    &= \mu - \Pow_{-q}(m_1^{(j)}, \ldots, m_T^{(j)}) \\
    &= \mu - \mu \left( \frac{1}{T} \sum_{t=1}^T a_{j,t}^{-q} \right)^{-1/q} 
        && \text{(factoring out $\mu$)} \\
    &= \mu \big( 1 - \mathcal{A}_j^{-1/q} \big)  \\
    &\ge \mu \big( 1 - (1 + c_q)^{-1/q} \big). 
        && \text{(since $\mathcal{A}_j \ge 1 + c_q$)}
\end{align*}
Defining the strictly positive constant $c_q' := 1 - (1 + c_q)^{-1/q}$, the regret is bounded by $c_q' \mu = c_q' \sigma \sqrt{k^q / T}$. Because $\Alg$ was chosen arbitrarily, taking the infimum over all algorithms completes the proof.
\end{proof}

\begin{proof}[\textbf{Proof of \Cref{thm:lower-combined}}]
Let $q = |p|$. \Cref{prop:avg-lb-main} bounds the regret by $\Omega(\sqrt{k/T})$ while \Cref{thm:lower-main-restated} bounds it by $\Omega(\sqrt{k^q/T})$. Taking the maximum of these two lower bounds and absorbing the respective constants establishes the result across all regimes.
\end{proof}
\subsection{An Information-Cost Principle}
\label{sec:info-cost}

We isolate the deterministic argument that underlies the lower bound. The statement
below abstracts the following tradeoff. A sequential procedure is summarized at round
\(t\) by a scalar success level \(s_t\). Achieving a larger value of \(s_t\) requires the
history to contain a corresponding amount of information about a hidden index, while
the information gained in one round is itself limited by the current value of \(s_t\).
Any objective that penalizes small values of \(s_t\) must therefore incur a cumulative
cost before \(s_t\) can reach a constant scale.

The lemma is independent of the particular bandit construction. The objective enters
only through a nonincreasing penalty function \(c\), while the statistical model enters
through an information requirement \(\phi\) and an upper bound \(\lambda\psi(s_t)\) on
the one-step information gain.

\begin{lemma}[Information-Cost Inequality]
\label{lem:info-cost}
Fix a baseline \(s_0\in(0,1)\) and a target level \(s^\star\in(s_0,1]\). Let
\((s_t)_{t=1}^T\) be a sequence in \([s_0,1]\), and let
\((\mathcal I_t)_{t=1}^{T+1}\) be a nondecreasing sequence with
\(\mathcal I_1=0\). Let
\begin{itemize}
    \item \(c:[s_0,1]\to\mathbb R_{\ge0}\) be nonincreasing;
    \item \(\phi:[s_0,1]\to\mathbb R_{\ge0}\) be absolutely continuous and strictly
    increasing, with \(\phi(s_0)=0\);
    \item \(\psi:[s_0,1]\to\mathbb R_{>0}\) be nondecreasing.
\end{itemize}
Suppose that, for some \(\lambda>0\), the following two inequalities hold for every
\(t\in[T]\):
\[
    \phi(s_t)\le \mathcal I_t,
    \qquad\qquad
    \mathcal I_{t+1}-\mathcal I_t\le \lambda\,\psi(s_t).
\]
Then
\[
    \sum_{t=1}^{T} c(s_t)
    \ \ge\
    \min\left\{
        c(s^\star)T,\,
        \frac1\lambda
        \int_{s_0}^{s^\star}
        \frac{c(s)\phi'(s)}{\psi(s)}\,ds
    \right\}.
\]
\end{lemma}

\begin{proof}
Let \(\overline{\mathcal I}:=\phi(s^\star)\), and define
\[
    \tau:=\inf\{t\in\{1,\ldots,T+1\}:\mathcal I_t\ge \overline{\mathcal I}\},
\]
with the convention that \(\tau=+\infty\) if the set is empty.

First suppose that \(\tau=+\infty\). Then for every \(t\in[T]\),
$\phi(s_t)\le \mathcal I_t<\phi(s^\star)$.
Since \(\phi\) is strictly increasing, this implies \(s_t<s^\star\). As \(c\) is
non-increasing, \(c(s_t)\ge c(s^\star)\) for all \(t\), and hence
\[
    \sum_{t=1}^T c(s_t)\ge c(s^\star)T.
\]

It remains to consider the case \(\tau<+\infty\). For every \(t<\tau\), we have
\(\mathcal I_t<\overline{\mathcal I}\), and therefore \(s_t<s^\star\). Define the
truncated increment
\[
    \delta_t
    :=
    \min\{\mathcal I_{t+1},\overline{\mathcal I}\}-\mathcal I_t,
    \qquad t=1,\ldots,\tau-1.
\]
Then \(\delta_t\ge0\), and $\delta_t
    \le
    \mathcal I_{t+1}-\mathcal I_t
    \le
    \lambda\psi(s_t)$.
Since \(c\ge0\) and \(\psi>0\), it follows that
\[
    c(s_t)
    \ge
    \frac1\lambda\,\frac{c(s_t)}{\psi(s_t)}\,\delta_t.
\]
Set $g(s):=\frac{c(s)}{\psi(s)}$.
Because \(c\) is nonincreasing and nonnegative, while \(\psi\) is positive and
nondecreasing, we have \(g\) is nonincreasing. The inequality
\(\phi(s_t)\le \mathcal I_t\), together with the monotonicity of \(\phi\), implies
\[
    s_t\le \phi^{-1}(\mathcal I_t).
\]
Thus $g(s_t)\ge g(\phi^{-1}(\mathcal I_t))$.
Define
\[
    f(u):=g(\phi^{-1}(u)),
    \qquad u\in[0,\overline{\mathcal I}].
\]
The function \(f\) is nonincreasing. Hence, for each \(t<\tau\), $c(s_t)
    \ge
    \frac1\lambda f(\mathcal I_t)\delta_t$.
    
The intervals $[\mathcal I_t,\mathcal I_t+\delta_t]$, $t=1,\ldots,\tau-1$, form a partition of \([0,\overline{\mathcal I}]\), up to endpoints of measure zero:
they start at \(\mathcal I_1=0\), each interval ends at the next information level
unless the threshold \(\overline{\mathcal I}\) is reached, and the final interval is
truncated exactly at \(\overline{\mathcal I}\). Since \(f\) is nonincreasing,
\[
    f(\mathcal I_t)\delta_t
    \ge
    \int_{\mathcal I_t}^{\mathcal I_t+\delta_t} f(u)\,du.
\]
Summing over \(t<\tau\), we obtain
\[
    \sum_{t<\tau} c(s_t)
    \ge
    \frac1\lambda
    \sum_{t<\tau} f(\mathcal I_t)\delta_t
    \ge
    \frac1\lambda
    \int_{0}^{\overline{\mathcal I}} f(u)\,du.
\]
Substituting \(u=\phi(s)\) gives
\[
    \int_{0}^{\overline{\mathcal I}} f(u)\,du
    =
    \int_{s_0}^{s^\star}
    \frac{c(s)}{\psi(s)}\phi'(s)\,ds.
\]
Therefore
\[
    \sum_{t=1}^T c(s_t)
    \ge
    \sum_{t<\tau} c(s_t)
    \ge
    \frac1\lambda
    \int_{s_0}^{s^\star}
    \frac{c(s)\phi'(s)}{\psi(s)}\,ds.
\]
Combining the two cases proves the claimed lower bound.
\end{proof}

\begin{corollary}[Negative-power instantiation; cf.\ Lemma~\ref{lem:det-cost}]
\label{cor:neg-power}
Fix \(q>0\), \(k\ge8\), and \(\lambda>0\). Let \((s_t)_{t=1}^T\) be a sequence in
\([0,1]\), and $r_t:=\max\{s_t,1/k\}$.
Suppose that a nondecreasing sequence \((\mathcal I_t)_{t=1}^{T+1}\), with
\(\mathcal I_1=0\), satisfies
\[
    d_k(r_t)\le \mathcal I_t,
    \qquad
    \mathcal I_{t+1}-\mathcal I_t\le \lambda r_t
    \qquad\text{for all }t\in[T],
\]
where \(d_k(s):=\operatorname{kl}(s\,\|\,1/k)\). Then there exists a constant
\(c_q>0\), depending only on \(q\), such that
\[
    \sum_{t=1}^{T}\bigl(r_t^{-q}-1\bigr)
    \ \ge\
    c_q\min\left\{T,\frac{k^q}{\lambda}\right\}.
\]
\end{corollary}

\begin{proof}
Apply Lemma~\ref{lem:info-cost} to the sequence \((r_t)_{t=1}^T\) with
\[
    s_0=\frac1k,
    \qquad
    s^\star=\frac12,
    \qquad
    c(s)=s^{-q}-1,
    \qquad
    \phi(s)=d_k(s),
    \qquad
    \psi(s)=s.
\]
The two assumed inequalities are precisely the two hypotheses of
Lemma~\ref{lem:info-cost}. The function \(c\) is nonincreasing and nonnegative on
\((0,1]\), the function \(d_k\) is strictly increasing on \([1/k,1]\), and
\(d_k(1/k)=0\). It remains to lower-bound the integral term. A direct computation gives
\[
    d_k'(s)
    =
    \log\left(\frac{s(k-1)}{1-s}\right).
\]
For \(s\in[2/k,4/k]\) and \(k\ge8\), we have
$d_k'(s)
    \ge
    \log\left(\frac{sk}{2}\right)$.
Moreover, since \(s\le1/2\) on the interval of integration, we have
\[
    c(s)=s^{-q}-1
    \ge
    (1-2^{-q})s^{-q}.
\]
Restricting the nonnegative integrand to the subinterval \([2/k,4/k]\), we obtain
\[
\begin{aligned}
    \frac1\lambda
    \int_{1/k}^{1/2}
    \frac{c(s)d_k'(s)}{s}\,ds
    &\ge
    \frac{1-2^{-q}}{\lambda}
    \int_{2/k}^{4/k}
    s^{-q-1}\log\left(\frac{sk}{2}\right)\,ds  \\
    &=
    \frac{1-2^{-q}}{\lambda}
    \left(\frac{k}{2}\right)^q
    \int_1^2 u^{-q-1}\log u\,du,
\end{aligned}
\]
where the last equality uses the substitution \(u=sk/2\). The integral
\[
    I_q:=\int_1^2 u^{-q-1}\log u\,du
\]
is a positive constant depending only on \(q\). Thus, the integral term is at least
\(c_q' k^q/\lambda\), where $c_q':=(1-2^{-q})2^{-q}I_q>0$.
The other term in Lemma~\ref{lem:info-cost} is $c(s^\star)T=(2^q-1)T$. Therefore
\[
    \sum_{t=1}^{T}\bigl(r_t^{-q}-1\bigr)
    \ge
    \min\left\{(2^q-1)T,\frac{c_q'k^q}{\lambda}\right\}
    \ge
    c_q\min\left\{T,\frac{k^q}{\lambda}\right\},
\]
where \(c_q:=\min\{2^q-1,c_q'\}\).
\end{proof}

\begin{remark}[Interpreting the information-cost inequality]
\label{rem:knobs}
Lemma~\ref{lem:info-cost} separates the lower-bound argument into a small number of
model-dependent components. The function \(c\) describes how the objective penalizes
a given success level. The function \(\phi\) describes how much information is needed
to sustain that success level. The baseline \(s_0\) is the success probability available
without information. Finally, the pair \((\lambda,\psi)\) describes the maximum rate
at which information can be accumulated in one round.
With these components specified, the cumulative cost is controlled by the single
integral $\frac1\lambda
    \int_{s_0}^{s^\star}
    \frac{c(s)\phi'(s)}{\psi(s)}\,ds$.
In the \(k\)-armed hard instance used here, \(\psi(s)=s\) because observations from
non-optimal arms have the same distribution under every hypothesis; information about
the hidden index is obtained only when the optimal arm is pulled
(Lemma~\ref{lem:info-gain}).
\end{remark}
\section{Algorithm and Upper Bound} \label{sec:algo}

In this section, we formally present our \hare\ algorithm and establish its theoretical guarantees.

\begin{algorithm}[htbp]
\caption{UCB with Harmonic Anchored Rank Exploration (\hare)}\label{alg:hare}
\small
\textbf{Input:} Number of arms $k$, time horizon $T$, Sub-Gaussian proxy $\sigma$, Confidence $\delta$.\\
Set $L = \log(8kT/\delta)$.
\begin{algorithmic}[1]
\State Initialize round $t \gets 1$; block $b \gets 1$; counts $N_i \gets 0$ and empirical means $\muhat_i \gets 0$ for all $i \in [k]$. 
\State Define $c(n) \triangleq \sigma\sqrt{2L/n}$  for $n \ge 1$ (with $c(0) = +\infty$).
\State Draw random permutation $\pi$ of $[k]$ and initialize schedule $(r_b)_{b \ge 1}$.
\State Define $B\triangleq \max_{i \in [k]} \{\muhat_i - c(N_i)\}$. 
\vspace{1.5mm}
\Statex \textit{\textcolor{gray}{\% Phase I: Preparation Phase}}
\While{$b \le \lceil T/2 \rceil$ \textbf{and not} $\big(B > 0 \land \min_{i} N_i \ge 32\sigma^2L/B^2\big)$}
    \State Toss fair coin $\theta_b \sim \Bern(1/2)$.
    \For{slot $h \in (\theta_b, 1-\theta_b)$}
        \State \textbf{if} $t > T$ \textbf{then break} \Comment{\textit{Halt if exact horizon reached}}
        \If{$h = 1$ \textbf{or} $B \le 0$}
            \State Pull scheduled arm $I_t = \pi_{r_b}$. \Comment{\textit{Scheduled slot or anchorless fallback}}
        \Else
            \State Pull active anchor $I_t \in \argmax_{i} \{\muhat_i - c(N_i)\}$.
        \EndIf
        \State Observe reward $R_t \sim \nu_{I_t}$, update $N_{I_t}, \muhat_{I_t}$, $B$ and increment $t \gets t+1$.
    \EndFor
    \State $b \gets b + 1$.
\EndWhile
\vspace{1.5mm}
\Statex \textit{\textcolor{gray}{\% Phase II: UCB Exploitation Phase}}
\While{$t \le T$}
    \State Pull arm $I_t \in \argmax_{i} \{\muhat_i + c(N_i)\}$ and observe reward $R_t \sim \nu_{I_t}$.
    \State Update $N_{I_t} \gets N_{I_t} + 1$ and $\muhat_{I_t} \gets \frac{(N_{I_t} - 1)\muhat_{I_t}}{N_{I_t}} + \frac{R_t}{N_{I_t}}$.
    \State $t \gets t + 1$.
\EndWhile
\end{algorithmic}
\end{algorithm}

The following theorem establishes the $p$-mean regret upper bound of \hare. Notice that the leading term replaces the $k^{q+1}$ dependency with the \textit{harmonic penalty sum} $K_q(k)$, matching the fundamental information-theoretic limits.

\begin{theorem}[Algorithm Upper Bound]\label{thm:upper-main}
Fix $q>0$ and let $p=-q$. On any $\sigma$-sub-Gaussian bandit instance with non-negative means, running \hare\ with confidence parameter $\delta \in (0, 1)$ satisfies:
\[
    R_{T,-q}(\Alg, \nu)
    \le
    C_q\left[
        \sigma\sqrt{\frac{H_k K_q(k) L}{T}}
        +\frac{\mustar H_k K_q(k)}{T}
        +\mustar\delta
    \right],
\]
where $L=\log(8kT/\delta)$ and $C_q>0$ depends only on $q$.
\end{theorem}

Instantiating the confidence parameter as $\delta=1/T$ sets the failure penalty to $\mustar/T$. For a sufficiently large horizon ($T \gtrsim \mustar^2 H_k K_q(k) / \sigma^2$), this $O(1/T)$ trailing term is dominated by the leading exploration cost. 

The asymptotic limits of $K_q(k)$ from \Cref{fact:harmonic-penalty} then classify the regret into three distinct regimes, closing the complexity gap.

\begin{corollary}[Consequences by Regime]\label{cor:upper-regimes}
For $p=-q < 0$, setting $\delta=1/T$ and absorbing logarithmic factors into the $\widetilde{O}$ notation yields:
\[
    R_{T,-q}(\Alg, \nu)
    \le
    \begin{cases}
        \widetilde{O}_q\left(\sigma\sqrt{\frac{k}{T}} + \frac{\mustar k}{T}\right), & 0<q<1,\\[8pt]
        \widetilde{O}\left(\sigma\sqrt{\frac{k}{T}} + \frac{\mustar k}{T}\right), & q=1,\\[8pt]
        \widetilde{O}_q\left(\sigma\frac{k^{q/2}}{\sqrt{T}} + \frac{\mustar k^q}{T}\right), & q>1.
    \end{cases}
\]
For $q > 1$, the leading $k^{q/2}$ factor matches the minimax lower bound of \Cref{thm:lower-combined} up to logarithmic terms.
\end{corollary}

Here, we note that the introductory bound presented in \Cref{thm:intro-upper} is simply an asymptotic restatement of \Cref{thm:upper-main} that suppresses the transient $O(1/T)$ failure penalty to highlight the dependence on $K_q(k)$.

\section{Proof of Upper Bound}\label{sec:upper-proof}

This section proves \Cref{thm:upper-main}.  Fix an instance with optimal mean $\mustar$.  If $\mustar=0$, then every mean is zero and the regret is zero.  Hence assume $\mustar>0$.

We first fix the standard ``good event'' that bounds the empirical mean estimator $\muhat_{i,n}$ for all arms and sample counts. 

Let $L = \log(8kT/\delta)$ and define the confidence radius $c(n) := \sigma\sqrt{2L/n}$. We define the good event as
\[
    \calE := \quad
    \left\{
    \forall i \in [k], \, \forall n \in [T] :
    |\muhat_{i,n} - \mu_i| \le c(n)
    \right\} 
    \quad = \quad 
    \bigcap_{i=1}^k \bigcap_{n=1}^T \Big\{ |\muhat_{i,n} - \mu_i| \le c(n) \Big\}.
\]

\begin{lemma}[Concentration]\label{lem:upper-good}
The good event holds with probability $\Pbb(\calE) \ge 1 - \delta/4$.
\end{lemma}

\begin{proof}
By Hoeffding's inequality and a union bound over the $kT$ combinations of arms and sample counts, the probability of the complementary failure event is strictly bounded by:
\[
    \Pbb(\calE^c) \le \sum_{i=1}^k \sum_{n=1}^T 2\exp\left(-\frac{n c(n)^2}{2\sigma^2}\right) = 2kT e^{-L}.
\]
Substituting $L = \log(8kT/\delta)$ yields $\Pbb(\calE^c) \le \delta/4$, and the result follows.
\end{proof}

Throughout this section, define the critical sample threshold as:
\[
    n_\star := \max\left(\left\lceil 128\frac{\sigma^2L}{\mustar^2}\right\rceil, 1\right).
\]

Under the good event, we now show that allocating $n_\star$ samples to the optimal arm is sufficient to establish an anchor, and hence, a positive welfare floor.

\begin{lemma}[Establishing the anchor]\label{lem:certification}
On the good event $\calE$, if the optimal arm has accumulated at least $n_\star$ samples, 
an anchor exists and any such anchor arm $a \in [k]$ has a true mean $\mu_a \ge 3\mustar/4$.
\end{lemma}

\begin{proof}

Assume current round is $t$ and $N_{i_\star}(t-1) \ge n_\star$. Because $n_\star \ge 128\sigma^2L/\mustar^2$, we have
\[
    c(N_{i_\star}(t-1)) \le \sigma\sqrt{\frac{2L}{128\sigma^2L/\mustar^2}} = \frac{\mustar}{8}.
\]
Under $\calE$, the empirical mean satisfies $\muhat_{i_\star}(t-1) \ge \mustar - c(N_{i_\star}(t-1))$. The lower confidence bound therefore yields
\[
    L_{i_\star}(t-1) := \muhat_{i_\star}(t-1) - c(N_{i_\star}(t-1)) \ge \mustar - 2c(N_{i_\star}(t-1)) \ge \mustar - \frac{\mustar}{4} = \frac{3}{4}\mustar.
\]
Recall that an anchor exists whenever the global safety threshold satisfies $B_{t-1} := \max_i L_i(t-1) > 0$. By definition of the maximum,
\[
    B_{t-1} \ge L_{i_\star}(t-1) \ge \frac{3}{4}\mustar > 0,
\]
guaranteeing the anchor. For any selected anchor arm $a \in \argmax_i L_i(t-1)$, the good event $\calE$ implies $\mu_a \ge \muhat_a(t-1) - c(N_a(t-1)) = L_a(t-1)$. We conclude that
\[
    \mu_a \ge L_a(t-1) = B_{t-1} \ge \frac{3}{4}\mustar. \qedhere
\]
\end{proof}

Set $S = \lceil H_kn_\star \rceil$. We define an expanding exploration frontier $a_b$ over $\pi$ that guarantees the sampling threshold $n_\star$ to all arms with ranks $r < a_b$.

\begin{lemma}[Expanding frontier]\label{lem:critical-rank}
For any block $b \ge 1$, define the frontier
\[
    a_b := \min\left(k, \left\lfloor\frac{b-1}{8 S}\right\rfloor\right).
\]
By the start of block $b$, every arm assigned a rank $r \le a_b$ is guaranteed at least $n_\star$ total pulls. Furthermore, for all blocks $b \ge 8kS + 1$, the frontier strictly saturates at $a_b = k$, guaranteeing that all $k$ arms have achieved this sampling threshold.
\end{lemma}

\begin{proof}
Assume $a_b \ge 1$. Consider any arm with rank $r \le a_b$. By \Cref{lem:prefix-balance}, the scheduled pulls accrued by this rank prior to block $b$ satisfy
\[
    N_{r}^{\rm sch}(b-1) \ge \frac{b-1}{r H_k} - 2 \ge \frac{b-1}{a_b H_k} - 2.
\]
Substituting $S = \lceil H_k n_\star \rceil$ into the frontier definition yields $a_b \le (b-1)/(8S)$, which rearranges to $(b-1)/(a_b H_k) \ge 8n_\star$. Because total pulls dominate scheduled pulls, the samples accumulated up to any round $t$ in block $b$ satisfy
\[
    N(t) \ge N_r^{\rm sch}(b-1) \ge 8n_\star - 2 \ge n_\star.
\]
For the saturation guarantee, consider any block $b \ge 8kS + 1$. Substituting the lower bound $b - 1 \ge 8kS$ into the definition of $a_b$ yields
\[
    a_b \ge \min\left(k, \left\lfloor\frac{8kS}{8S}\right\rfloor\right) = k.
\]
Since $a_b \le k$ by definition, this enforces $a_b = k$, ensuring the sampling guarantee applies to all $k$ arms.
\end{proof}

Together, \Cref{lem:certification} and \Cref{lem:critical-rank} govern the discovery of the anchor: it is established early if $R_\star \le a_b$, and guaranteed in any case by block $8kS+1$, under $\calE$. We now use this to lower bound the expected reward $m_t$ for any round $t$.

Recall that the stopping condition for Phase I is evaluated exclusively at the end of each two-round block (i.e., at even rounds). Specifically, Phase I terminates at the first even round $t$ where the anchor is secured and the sample counts satisfy:
\[
    B_t > 0 \quad \text{and} \quad \min_{i \in [k]} N_i(t) \ge \frac{32 \sigma^2 L}{B_t^2}.
\]
Let $\tau \in \{2, 4, 6, \dots\}$ denote this random stopping time in rounds. We now prove that this empirical threshold guarantees strict safety for Phase II.

\begin{lemma}[Phase II safety]\label{lem:phase2-safe}
On the good event $\mathcal{E}$, if Phase I terminates at even round $\tau$, the following hold:
\begin{enumerate}
    \item $B_\tau \ge 2\mustar/3$. \label{claim:cl1}
    \item For all $i \in [k]$, $N_i(\tau) \ge 32\sigma^2L/\mustar^2$. \label{claim:cl2}
    \item In every Phase II round $t > \tau$, the UCB-selected arm satisfies $\mu_i \ge \mustar/2$. \label{claim:cl3}
\end{enumerate}
\end{lemma}

\begin{proof}
Let $\tau$ denote the stopping round. On $\mathcal{E}$, the stopping condition $N_{i_\star}(\tau) \ge 32\sigma^2L/B_\tau^2$ bounds the optimal arm's confidence radius:
\[
    c(N_{i_\star}(\tau)) \le \sigma\sqrt{\frac{2L}{32\sigma^2L/B_\tau^2}} = \frac{B_\tau}{4}.
\]
Under $\mathcal{E}$, the empirical mean satisfies $\hat{\mu}_{i_\star}(\tau) \ge \mustar - c(N_{i_\star}(\tau))$. The global safety threshold $B_\tau := \max_j L_j(\tau)$ therefore satisfies:
\[
    B_\tau \ge L_{i_\star}(\tau) = \hat{\mu}_{i_\star}(\tau) - c(N_{i_\star}(\tau)) \ge \mustar - 2c(N_{i_\star}(\tau)) \ge \mustar - \frac{B_\tau}{2}.
\]
Rearranging yields $\frac{3}{2}B_\tau \ge \mustar$, establishing (\ref{claim:cl1}).

Because no lower confidence bound can exceed its true mean on $\mathcal{E}$, we have $B_\tau \le \mustar$. The stopping condition thus enforces a universal minimum sample count for all $i \in [k]$:
\[
    N_i(\tau) \ge \frac{32\sigma^2L}{B_\tau^2} \ge \frac{32\sigma^2L}{\mustar^2},
\]
establishing (\ref{claim:cl2}). Consequently, the confidence radius for every arm $i$ (sampled $N_i$ times) throughout Phase II is bounded as $c(N_i) \le \mustar/4$.

For (\ref{claim:cl3}), let $i$ be the arm selected by the UCB index in any Phase II round $t > \tau$. On $\mathcal{E}$, its index dominates that of the optimal arm:
\[
    \hat{\mu}_i(t-1) + c(N_i(t-1)) \ge \hat{\mu}_{i_\star}(t-1) + c(N_{i_\star}(t-1)) \ge \mustar.
\]
Applying the lower confidence bound $\mu_i \ge \hat{\mu}_i(t-1) - c(N_i(t-1))$ yields
\begin{equation}
    \mustar - \mu_i \;\le\; 2\,c(N_i(t-1)). \label{eq:ucb-sandwich}
\end{equation}
Since $c(N_i(t-1)) \le \mustar/4$ by (\ref{claim:cl2}), \eqref{eq:ucb-sandwich} gives
\[
    \mu_i \;\ge\; \mustar - 2\left(\frac{\mustar}{4}\right) \;=\; \frac{\mustar}{2}. \qedhere
\]
\end{proof}

\addtocounter{equation}{-1}

\begin{lemma}[Ex-ante reward lower bound]\label{lem:ex-ante-reward}
There exists a universal constant $c_0 > 0$ such that for every round $t \le T$, with $b := \lceil t/2 \rceil$,
\[
    m_t := \E[\mu_{I_t}] \ge c_0 \mustar \max\left( \frac{1}{k}, \frac{a_b}{k} \right).
\]
\end{lemma}

\begin{proof}

Let $R_\star \in [k]$ denote the rank of the optimal arm under the initial permutation $\pi$. We fix the following probability space: for each arm $i \in [k]$, draw an i.i.d.\ sequence of rewards $X_{i,1}, X_{i,2}, \dots$ from $\nu_i$ in advance, so that the $n$-th time arm $i$ is pulled, the algorithm simply reveals $X_{i,n}$. Note that these reward sequences, the permutation $\pi$, and the block coins $(\theta_b)_{b\ge1}$ are all drawn independently of one another. Three consequences follow immediately, and are all we need:
\begin{itemize}
    \item Since every empirical mean $\hat\mu_{i,n}$, and hence the good event $\calE$, depends only on the reward sequences, $\calE$ is independent of $\pi$ and of every $\theta_b$.
    \item Since $R_\star$ depends only on $\pi$, any event defined through $R_\star$ (in particular $\{R_\star = r_b\}$ and $\{R_\star \le a_b\}$, used below) is independent of the rewards and of every $\theta_b$.
    \item Combining the two, $\calE$ and any event defined through $R_\star$ are independent of each other, since the first depends only on rewards and the second only on $\pi$.
\end{itemize}

Let $\tau$ denote the stopping round at which the algorithm transitions from Phase I to Phase II, and let $\mathcal{H}_{2b-2}$ denote the history generated by $\pi$, the rewards revealed up to round $2b-2$, and $\theta_1,\dots,\theta_{b-1}$. Because Phase I termination~\eqref{eq:stop-condition} is evaluated exclusively at block boundaries, the indicator $\{t \le \tau\}$ is $\mathcal{H}_{2b-2}$-measurable, and since $\theta_b$ is independent of everything $\mathcal{H}_{2b-2}$ depends on, it is independent of $\{t\le\tau\}$ too. The independence of $\theta_b$ from $\calE$ and from $\{R_\star = r_b\}$ was already noted above and does not rely on $\mathcal{H}_{2b-2}$ at all, since both events can depend on rewards beyond round $2b-2$.

At the onset of each Phase I block, the coin $\theta_b \sim \Bern(1/2)$ determines the order of the slots. Since $\theta_b$ is independent of $\calE$, $\{R_\star=r_b\}$, $\{R_\star\le a_b\}$, and $\{t\le\tau\}$, the marginal probability that any Phase I round $t\in\{2b-1,2b\}$ acts as the scheduled slot is exactly $1/2$ ; and remains $1/2$ even after conditioning on any of these events, or their intersections. The same holds for the auxiliary slot.

During Phase I ($t \le \tau$), the algorithm's action at round $t$ falls into one of three disjoint cases:
\begin{enumerate}
    \item Round $t$ is the scheduled slot. The algorithm deterministically pulls $\pi_{r_b}$.
    \item Round $t$ is the auxiliary slot and $B_{t-1} > 0$. The algorithm pulls an anchor arm.
    \item Round $t$ is the auxiliary slot and $B_{t-1} \le 0$. The algorithm defaults to pulling arm $\pi_{r_b}$.
\end{enumerate}

We lower-bound the expected reward $m_t$ by evaluating the algorithm under two sub-events of $\calE$. Assuming $\delta \le 1$, we have $\Pbb(\calE) \ge 1 - \delta/4 \ge 3/4$. We define the events $A_t$ and $F_t$ as follows:
\[
    A_t := \calE \cap \{R_\star = r_b\} \quad \text{and} \quad F_t := \calE \cap \{R_\star \le a_b\}.
\]
Note that we define $b := \lceil t/2 \rceil$ globally. This and the infinitude of our \textit{harmonic schedule} preserves the event definitions across both phases, despite the algorithm abandoning the block structure in Phase II.

By the independence noted above, their probabilities satisfy:
\begin{align*}
    \Pbb(A_t) &= \Pbb(\calE)\Pbb(R_\star = r_b) \ge \left(\frac{3}{4}\right)\frac{1}{k} = \frac{3}{4k}, \\
    \Pbb(F_t) &= \Pbb(\calE)\Pbb(R_\star \le a_b) \ge \left(\frac{3}{4}\right)\frac{a_b}{k} = \frac{3a_b}{4k}.
\end{align*}

Since $A_t, F_t \subseteq \calE$, Phase II's safety guarantee (\Cref{lem:phase2-safe}), which relies exclusively on $\calE$, applies unchanged. Thus, for the UCB regime ($t > \tau$), both $\E[\mu_{I_t} \mid A_t, t > \tau]$ and $\E[\mu_{I_t} \mid F_t, t > \tau]$ are lower-bounded by $\mustar/2$. As $\mu_i \ge 0$, we bound unfavorable Phase I cases strictly by zero.

First, we condition on $A_t$. For $t \le \tau$ (Phase I), we isolate Case 1. With probability $1/2$, round $t \in \{2b-1, 2b\}$ acts as the scheduled slot, forcing $I_t = i_\star$ under $A_t$. Bounding Cases 2 and 3 by zero yields:
\begin{align*}
    \E[\mu_{I_t} \mid A_t, t \le \tau] 
    &\ge \Pbb(\text{Case 1}) (\mustar) + \Pbb(\text{Cases 2 \& 3})(0) \\
    &= \frac{1}{2}(\mustar) = \frac{\mustar}{2}.
\end{align*}
Because the conditional expectation is at least $\mustar/2$ in both phases, we get $\E[\mu_{I_t} \mid A_t] \ge \mustar/2$ irrespective of the stopping time $\tau$. Removing the conditioning provides the first bound:
\[
    m_t \ge \E[\mu_{I_t} \mid A_t]\Pbb(A_t) \ge \frac{\mustar}{2}\left(\frac{3}{4k}\right) = \frac{3\mustar}{8k}.
\]

Next, we condition on $F_t$. For $t \le \tau$, we isolate Case 2. Under $F_t$, the constraint $R_\star \le a_b$ guarantees the optimal arm has received $n_\star$ samples prior to block $b$ (\Cref{lem:critical-rank}). Under $\calE$, this triggers the anchor $B_t \ge 3\mustar/4$ (\Cref{lem:certification}) and strictly eliminates Case 3. With probability $1/2$, round $t \in \{2b-1, 2b\}$ acts as the \textit{Auxiliary Slot} (Case 2) and pulls the anchor. Bounding Case 1 by zero yields:
\begin{align*}
    \E[\mu_{I_t} \mid F_t, t \le \tau] 
    &\ge \Pbb(\text{Case 1})(0) + \Pbb(\text{Case 2})\left(\frac{3\mustar}{4}\right) \\
    &= \frac{1}{2}\left(\frac{3\mustar}{4}\right) = \frac{3\mustar}{8}.
\end{align*}
Again, since Phase II guarantees $\mustar/2 \ge 3\mustar/8$, this lower bound holds irrespective of the stopping time $\tau$, yielding $\E[\mu_{I_t} \mid F_t] \ge 3\mustar/8$. Removing the conditioning provides the second bound:
\[
    m_t \ge \E[\mu_{I_t} \mid F_t]\Pbb(F_t) \ge \frac{3\mustar}{8}\left(\frac{3a_b}{4k}\right) = \frac{9a_b\mustar}{32k}.
\]

Combining the two bounds, the ex-ante reward satisfies:
\[
    m_t \ge \max\left(\frac{12\mustar}{32k}, \frac{9a_b\mustar}{32k}\right) \ge \frac{9}{32} \mustar \max\left(\frac{1}{k}, \frac{a_b}{k}\right).
\]
Setting $c_0 = 1/4 \le 9/32$ completes the proof.
\end{proof}

We now bound the maximum duration of Phase I and then derive a cap on its penalty.

\begin{lemma}[Deterministic preparation window]\label{lem:prep-ends}
On the good event $\mathcal{E}$, the random stopping time $\tau$ of Phase I is bounded as
\[
    \tau \le T_0 := 18 k S,
\]
where $S := \lceil H_k n_\star \rceil$.
\end{lemma}

\begin{proof}
By \Cref{lem:critical-rank}, at the end of block $b_0 := 8kS + 1$ (corresponding to round $2b_0$), the \textit{harmonic schedule} guarantees at least $n_\star$ total pulls to every arm $i \in [k]$. Consequently, the optimal arm satisfies $N_{i_\star}(2b_0) \ge n_\star$. On event $\mathcal{E}$, \Cref{lem:certification} thus ensures $B_{2b_0} \ge 3\mustar/4 > 0$.

As $n_\star := \lceil 128\sigma^2L/\mustar^2 \rceil$, we therefore have
\[
    \frac{32\sigma^2L}{B_{2b_0}^2} \le \frac{32\sigma^2L}{(3\mustar/4)^2} = \frac{512}{9}\frac{\sigma^2L}{\mustar^2} \le n_\star.
\]
Since $n_\star \le N_i(2b_0)$ for all $i \in [k]$, the stopping condition \eqref{eq:stop-condition} holds. Because $2b_0$ is an even round, the algorithm explicitly evaluates the condition here, forcing Phase I to end at $\tau$ where
\[
    \tau \le 2b_0 = 16kS + 2 \le 18kS =: T_0,
\]
and the final inequality holds since $k, S \ge 1$.
\end{proof}

\begin{lemma}[Preparation inverse penalty]\label{lem:prep-penalty}
There exists a constant $C_q > 0$ such that the total inverse penalty up to the deterministic horizon $T_0$ is bounded as
\[
    \sum_{t=1}^{T \wedge T_0} \left[ \left(\frac{\mu_*}{m_t}\right)^q - 1 \right] < C_q S K_q(k), 
\]
where $C_q := 34 \cdot 4^q$, $S := \lceil H_k n_\star \rceil$, and $K_q(k) := \sum_{r=1}^k \left(\frac{k}{r}\right)^q$.
\end{lemma}

\begin{proof}
By \Cref{lem:prep-ends}, Phase I ends by round $T_0 = 18kS$. Since each block spans two rounds, the timeline $t \le T \wedge T_0$ covers at most $9kS$ blocks. Converting the round-wise summation to a block-wise summation, we get
\[
    \sum_{t=1}^{T \wedge T_0} \left[ \left(\frac{\mu_*}{m_t}\right)^q - 1 \right]
    \le 2 \sum_{b=1}^{9kS} \left[ \left(\frac{\mu_*}{m_t}\right)^q - 1 \right]
    < 2 \sum_{b=1}^{9kS} \left(\frac{\mu_*}{m_t}\right)^q,
\]
where we relax the $-1$ term-wise in the final inequality. This switch is allowed, since
by \Cref{lem:ex-ante-reward}, the ex-ante reward satisfies $m_t \ge \frac{1}{4}\mu_* \max(1/k, a_b/k)$ for any round $t$ in block $b$, where $a_b = \min(k, \lfloor \frac{b-1}{8S} \rfloor)$. This implies:
\[
    \left(\frac{\mu_*}{m_t}\right)^q \le 4^q \max\left(\frac{1}{k}, \frac{a_b}{k}\right)^{-q}.
\]

For $b \in [1, 8S]$, we have $a_b = 0$, giving $\max(1/k, 0) = 1/k$. There are $8S$ such blocks.\\
For each $r \in \{1, 2, \dots, k-1\}$, we have $a_b = r$ over the $8S$ blocks in the interval $b \in [8rS + 1, 8(r+1)S]$.\\
For $b \in [8kS + 1, 9kS]$, the frontier saturates at $k$. This final interval contains $kS$ blocks.

Bounding the total penalty over these intervals yields
\begin{align*}
    \sum_{t=1}^{T \wedge T_0} \left[ \left(\frac{\mu_*}{m_t}\right)^q - 1 \right] 
    &< 2 \sum_{b=1}^{9kS} \left(\frac{\mu_*}{m_t}\right)^q \\
    &\le 2 \left[\sum_{b=1}^{8S} (4k)^q + \sum_{r=1}^{k-1} \sum_{b=8rS+1}^{8(r+1)S} 4^q \left(\frac{k}{r}\right)^q + \sum_{b=8kS+1}^{9kS} 4^q \left(\frac{k}{k}\right)^q \right] \\
    &= 2 \left[8S \cdot 4^q k^q + 8S \cdot 4^q \sum_{r=1}^{k-1} \left(\frac{k}{r}\right)^q + kS \cdot 4^q \right] \\
    &< 2 \left[8S \cdot 4^q K_q(k) + 8S \cdot 4^q K_q(k) + S \cdot 4^q K_q(k) \right] \\
    &= 34 \cdot 4^q S K_q(k).
\end{align*}
The last inequality follows because $k^q, \sum_{r=1}^{k-1} (k/r)^q$, and $k$ are all upper-bounded by $K_q(k)$.

Setting $C_q = 34 \cdot 4^q$ concludes the proof.
\end{proof}


We now bound the \textit{inverse penalty} for the remaining duration of the horizon.

\begin{lemma}[Phase II inverse penalty]\label{lem:phase2-cost}
There exists a constant $C_q > 0$ such that
\[
    \sum_{t=T_0 + 1}^{T} \left[\left(\frac{\mustar}{m_t}\right)^q-1\right]
    \le C_q\left(\frac{\sigma}{\mustar}\sqrt{kTL} + T\delta\right),
\]
where the sum evaluates to zero if $T \le T_0$.
\end{lemma}

\begin{proof}
If $T \le T_0$, the summation is empty and the bound holds trivially. Assume $T > T_0$. By \Cref{lem:prep-ends}, the stopping time satisfies $\tau \le T_0$, ensuring all rounds $t > T_0$ execute exclusively in Phase II. The stopping condition~\eqref{eq:stop-condition} requires $\min_i N_i(\tau) \ge 32\sigma^2 L/B_\tau^2 > 0$, and since $N_i(\tau)$ is a nonnegative integer this forces $N_i(\tau) \ge 1$ for every arm $i \in [k]$.

On event $\calE$, \Cref{lem:phase2-safe} guarantees the UCB index selects an arm $I_t \in \arg\max_i\{\hat\mu_i(t-1) + c(N_i(t-1))\}$ satisfying $\mu_{I_t} \ge \mustar/2$.

Ex-ante rewards are non-negative, so the Law of Total Expectation bounds them as
\[
    m_t = \E[\mu_{I_t}] 
    \ge \E[\mu_{I_t} \mid \calE]\Pbb(\calE) 
    \ge \frac{\mustar}{2} \Pbb(\calE).
\]
By \Cref{lem:upper-good}, $\Pbb(\calE) \ge 1 - \delta/4 \ge 1/2$. Thus, $m_t \ge \mustar/4$. 

Since $m_t \le \mustar$ unconditionally, we have $m_t \in [\mustar/4, \mustar]$. Applying \Cref{lem:lipschitz-inverse} bounds the inverse penalty by the linear regret:
\[
    \left(\frac{\mustar}{m_t}\right)^q - 1 \le C_q' \frac{\mustar - m_t}{\mustar}.
\]

By linearity of expectation ($\mustar - m_t = \E[\mustar - \mu_{I_t}]$) and the Law of Total Expectation, we decompose the single-round regret. As the maximum possible regret is bounded by $\mustar$:
\begin{align*}
    \mustar - m_t
    &= \E[\mustar - \mu_{I_t} \mid \calE]\Pbb(\calE) + \E[\mustar - \mu_{I_t} \mid \calE^c]\Pbb(\calE^c) \\
    &\le \E[\mustar - \mu_{I_t} \mid \calE] + \mustar \Pbb(\calE^c).
\end{align*}

Applying linearity of expectation again, we get
\begin{align}
    \sum_{t=T_0 + 1}^{T} \left[\left(\frac{\mustar}{m_t}\right)^q - 1\right]
    &\le \frac{C_q'}{\mustar} \sum_{t=T_0 + 1}^{T} \Big( \E[\mustar - \mu_{I_t} \mid \calE] + \mustar \Pbb(\calE^c) \Big) \nonumber \\
    &\le \frac{C_q'}{\mustar} \E\left[ \sum_{t=T_0 + 1}^{T} (\mustar - \mu_{I_t}) \;\middle|\; \calE \right] + C_q' T \delta, \label{eq:phase2-decomp}
\end{align}
where the second term applies the union bound $\Pbb(\calE^c) \le \delta/4 \leq \delta$ and the relaxation $T - T_0 < T$.

It remains to bound the inner sum on $\calE$. By \Cref{lem:prep-ends}, Phase I stops at $\tau \le T_0$, so it suffices to bound the larger sum over all rounds $t > \tau$. Fix any such round $t$; since $I_t$ maximizes the UCB index at round $t$, we have
\begin{equation*}
    \mustar - \mu_{I_t} \;\le\; 2\,c(N_{I_t}(t-1)), \qquad t > \tau \tag*{(recall \eqref{eq:ucb-sandwich})}
\end{equation*}
on $\calE$, regardless of how $N_{I_t}(t-1)$ was accumulated prior to $\tau$.

Fix an arm $i \in [k]$ and let $n_i := N_i(T) - N_i(\tau) \ge 0$ denote its number of pulls after round $\tau$. We track, for each arm $i$, the sum
\[
\sum_{t > \tau:\, I_t = i} c(N_i(t-1)),
\]
which ranges over exactly these $n_i$ rounds, namely those after $\tau$ at which arm $i$ is selected. At its $j$-th such pull, arm $i$ has already accumulated at least $N_i(\tau) + (j-1) \ge j$ samples (using $N_i(\tau)\ge 1$). Therefore, as $c(\cdot)$ is decreasing,
\[
\sum_{t > \tau:\, I_t = i} c(N_i(t-1)) \;\le\; \sum_{j=1}^{n_i} c(j) \;=\; \sigma\sqrt{2L}\sum_{j=1}^{n_i}\frac{1}{\sqrt j} \;\le\; 2\sigma\sqrt{2L}\,\sqrt{n_i}.
\]
For the last step, since $x \mapsto x^{-1/2}$ is decreasing, $j^{-1/2} \le \int_{j-1}^{j} x^{-1/2}\,dx$ for every $j \ge 2$. Summing from $j=2$ to $n_i$ and adding back the $j=1$ term,
\[
\sum_{j=1}^{n_i} \frac{1}{\sqrt j} \;\le\; 1 + \int_{1}^{n_i} x^{-1/2}\,dx \;=\; 2\sqrt{n_i} - 1 \;\le\; 2\sqrt{n_i},
\]
which also holds trivially when $n_i = 0$.

Summing over arms and applying Cauchy-Schwarz,
\[
\sum_{t=\tau+1}^{T} (\mustar - \mu_{I_t}) 
\;\le\; 2\sum_{i=1}^k \sum_{t>\tau:\,I_t=i} c(N_i(t-1)) 
\;\le\; 4\sigma\sqrt{2L}\sum_{i=1}^k \sqrt{n_i} 
\;\le\; 4\sigma\sqrt{2L}\,\sqrt{k\textstyle\sum_i n_i}
\;=\; 4\sigma\sqrt{2L}\,\sqrt{k(T-\tau)}.
\]
Since $0 \le \tau \le T_0$ and every term $\mustar - \mu_{I_t}$ is non-negative, this becomes
\[
\sum_{t=T_0+1}^{T}(\mustar - \mu_{I_t}) \;\le\; \sum_{t=\tau+1}^{T}(\mustar-\mu_{I_t}) \;\le\; 4\sigma\sqrt{2L}\,\sqrt{k(T-\tau)} \;\le\; 4\sigma\sqrt{2L}\,\sqrt{kT}.
\]
This holds on all of $\calE$, so taking conditional expectation given $\calE$ yields
\[
\E\left[\sum_{t=T_0+1}^{T}(\mustar-\mu_{I_t}) \;\middle|\; \calE\right] \;\le\; 4\sigma\sqrt{2L}\,\sqrt{kT},
\]
and substituting into~\eqref{eq:phase2-decomp} gives
\[
\sum_{t=T_0 + 1}^{T} \left[\left(\frac{\mustar}{m_t}\right)^q - 1\right] \;\le\; 4\sqrt{2}\,C_q' \frac{\sigma}{\mustar} \sqrt{kTL} + C_q' T \delta.
\]
Absorbing the scalar combinations into $C_q$ concludes the proof.
\end{proof}

\begin{proof}[\textbf{Proof of \Cref{thm:upper-main}}]
Recall that \Cref{lem:inverse-conversion} bounds the $(-q)$-mean regret by the inverse penalty. Splitting the summation at the fixed round $T_0$ yields:
\[
    R_{T,-q} \le \frac{\mustar}{qT} \left( \sum_{t=1}^{T \wedge T_0} \left[\left(\frac{\mustar}{m_t}\right)^q-1\right] + \sum_{t=T_0+1}^{T} \left[\left(\frac{\mustar}{m_t}\right)^q-1\right] \right).
\]
Substituting the bounds from \Cref{lem:prep-penalty} and \Cref{lem:phase2-cost}, and recalling $S = H_k n_\star$, leads to
\begin{align*}
    R_{T,-q}
    &\le \frac{\mustar}{qT} \left[ C_q H_k n_\star K_q(k) + C_q \left(\frac{\sigma}{\mustar}\sqrt{kTL} + T\delta\right) \right] \\
    &\le \frac{C_q'}{T} \left[ H_k K_q(k) (\mustar n_\star) + \sigma\sqrt{kTL} + \mustar T\delta \right].
\end{align*}

By definition, the anchor threshold is $n_\star = \max(\lceil 128\sigma^2L/\mustar^2 \rceil, 1) \le 128\sigma^2L/\mustar^2 + 2$, which gives
\[
    \mustar n_\star \le \frac{128\sigma^2L}{\mustar} + 2\mustar.
\]
The regret is thus bounded as
\begin{equation}\label{eq:pre-final-upper}
    R_{T,-q} \le C_q'' \left[ \frac{H_k K_q(k)}{T} \frac{\sigma^2L}{\mustar} + \mustar \frac{H_k K_q(k)}{T} + \sigma\sqrt{\frac{kL}{T}} + \mustar\delta \right].
\end{equation}

We bound \eqref{eq:pre-final-upper} by evaluating it relative to $\mu_0 := \sigma\sqrt{\frac{H_k K_q(k)L}{T}}$.

Case 1 ($\mustar \le \mu_0$): Since $\mu_i \ge 0$ for all $i \in [k]$, we have $m_t \ge 0$. The $(-q)$-mean regret therefore satisfies $R_{T,-q} \le \mustar \le \mu_0$.

Case 2 ($\mustar > \mu_0$): By definition of $\mu_0$, the leading term in \eqref{eq:pre-final-upper} is bounded as
\[
    \frac{H_k K_q(k)}{T} \frac{\sigma^2L}{\mustar} = \frac{\mu_0^2}{\mustar} < \mu_0.
\]
Moreover, as $H_k \ge 1$ and $K_q(k) \ge k$, we eventually have $\sigma\sqrt{kL/T} \le \mu_0$.

Combining both cases, the regret uniformly satisfies
\[
    R_{T,-q} \le C \left[ \mu_0 + \mustar \frac{H_k K_q(k)}{T} + \mustar\delta \right].
\]
Substituting $\mu_0$ concludes the proof:
\[
    R_{T,-q} \le C \left[ \sigma\sqrt{\frac{H_k K_q(k)L}{T}} + \mustar \frac{H_k K_q(k)}{T} + \mustar\delta \right].
\]
\end{proof}

\begin{proof}[\textbf{Proof of \Cref{cor:upper-regimes}}]
From \Cref{fact:harmonic-penalty}, we get
\[
    K_q(k) = 
    \begin{cases}
        \Theta_q(k), & 0 < q < 1, \\
        \Theta(k \log k), & q = 1, \\
        \Theta_q(k^q), & q > 1.
    \end{cases}
\]
Substituting these evaluations into the $O\left(\sigma\sqrt{\frac{H_k K_q(k) L}{T}}\right)$ leading term of \Cref{thm:upper-main}, and noting that the schedule penalty contributes $H_k = \Theta(\log k)$, produces the required bounds. Absorbing logarithmic terms into the $\widetilde{O}$ notation yields the three distinct regret regimes, concluding the proof.
\end{proof}

\section{Experiments}\label{sec:expt}

We evaluate \hare\ against \textsf{Welfarist UCB}~\cite{sarkar2025welfarist} and \textsf{Explore-Then-UCB}~\cite{krishna2025pmean} in four parts: a head-to-head regret comparison, a stress test across fairness levels, a closer look at exploration and a targeted test of regret scaling with the number of arms.
For the first two, we simulate a $k$-armed bandit with $k=50$, mean rewards $\mu_i$ drawn
i.i.d.\ uniformly from $[10,1000]$, and Gaussian noise variance $\sigma^2=400$, over a horizon $T=10^6$, averaging results over $50$ independent runs to estimate $\mathbb{E}[\mu_{I_t}]$.


\paragraph{Experiment A: Regret across algorithms.} \Cref{fig:h2h-p05,fig:h2h-p2,fig:h2h-p20} compare regret for $p \in \{-0.5,-2.0, -20\}$. \hare\ minimizes regret most rapidly. \textsf{Welfarist-UCB} is delayed by uniform exploration, and \textsf{Explore-Then-UCB} accrues linear regret prior to commitment. \hare\ mitigates both inefficiencies via a real-time lower confidence bound anchor.

\paragraph{Experiment B: Regret scales with $q = -p$.} \Cref{fig:q-sweep} shows regret across fairness levels $p \in \{-1, -5, -10, -20\}$, increasing in magnitude as $q =-p$ grows, consistent with our theoretical rates. \hare\ 's regret remains stable across this entire range.

\begin{figure*}[!h]
    \centering
    \captionsetup[subfigure]{font=small,labelfont=bf}
    \begin{subfigure}[b]{0.25\textwidth}
        \centering
        \includegraphics[width=\textwidth]{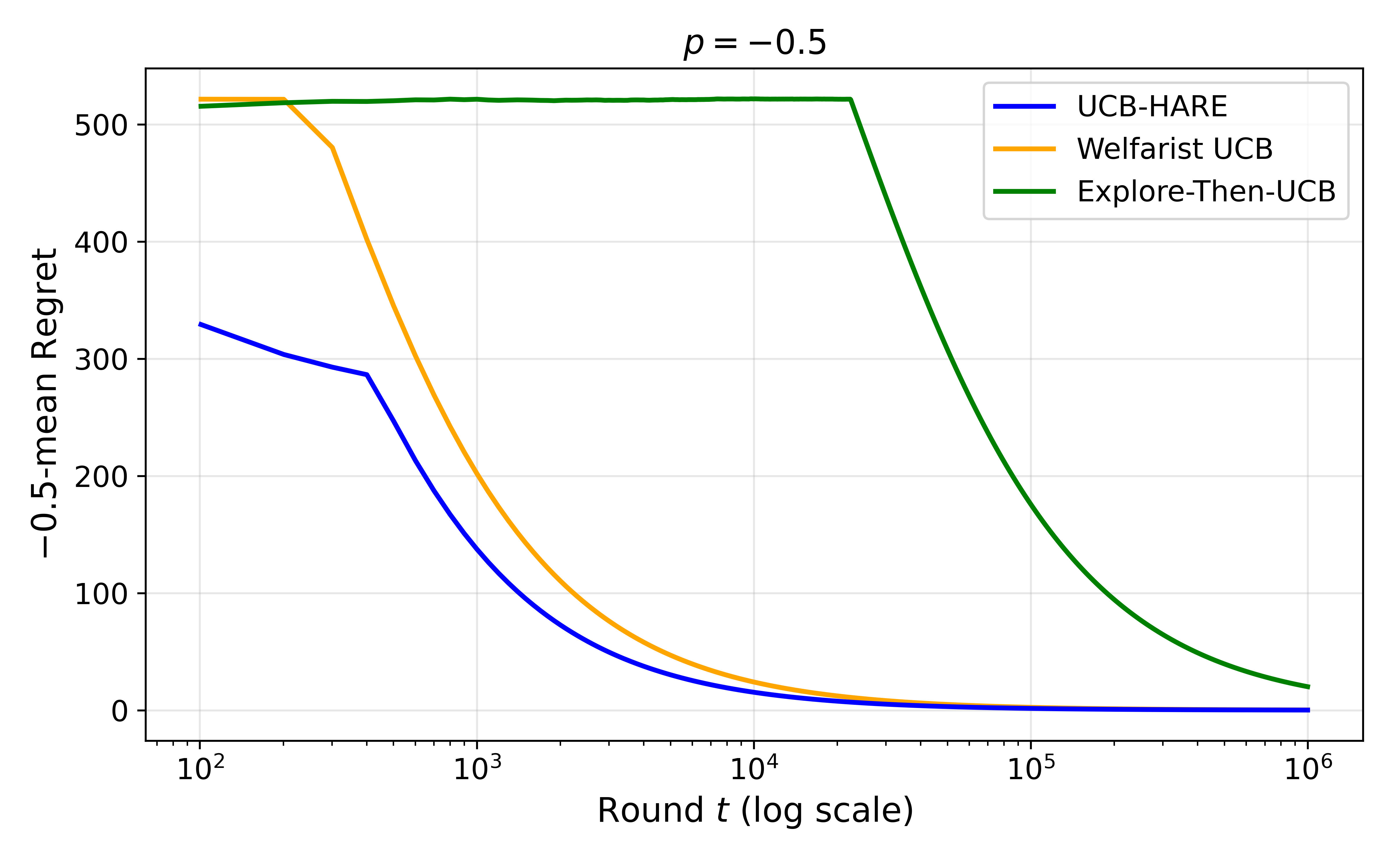}
        \caption{$p=-0.5$}
        \label{fig:h2h-p05}
    \end{subfigure}\hfill
    \begin{subfigure}[b]{0.25\textwidth}
        \centering
        \includegraphics[width=\textwidth]{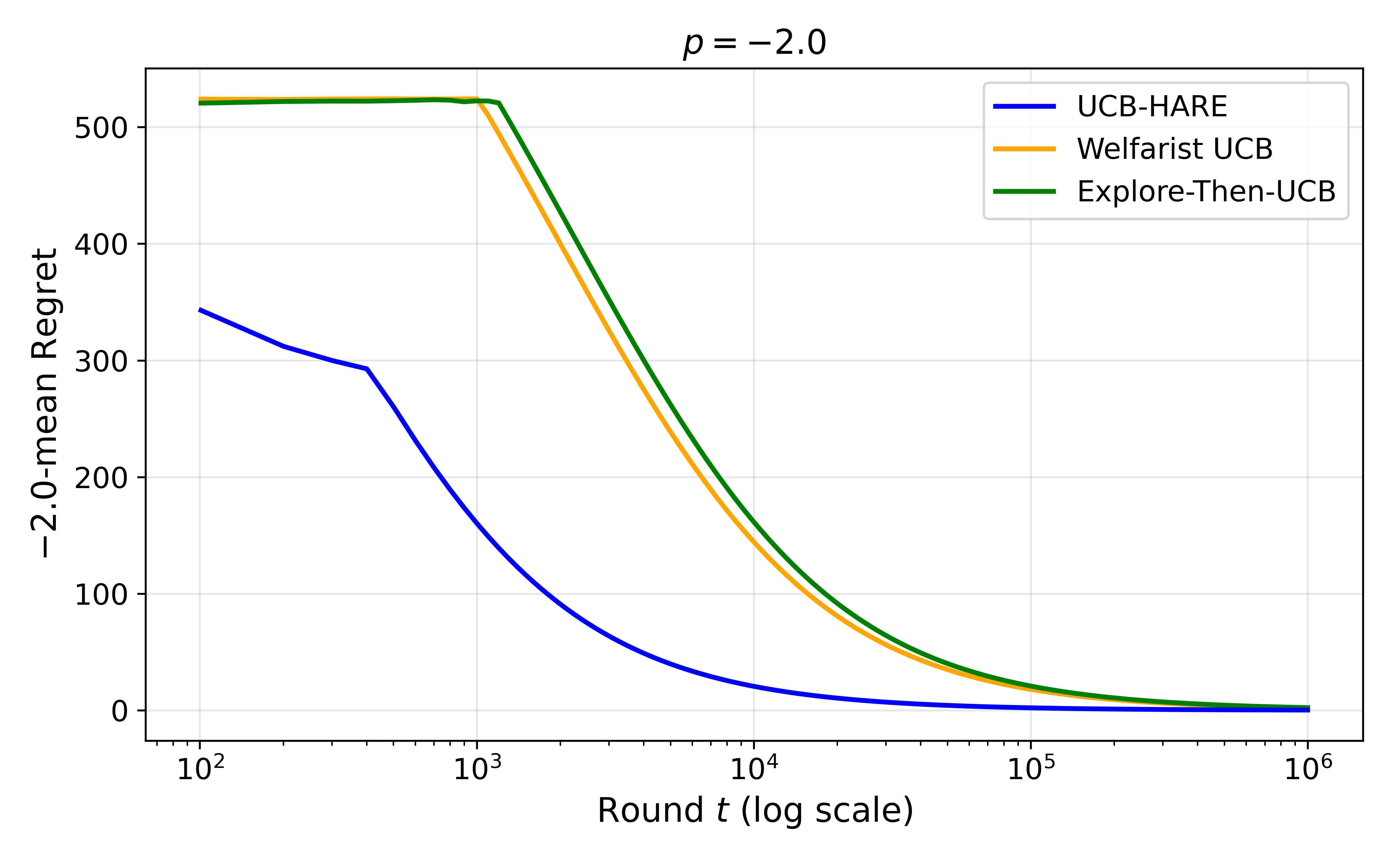}
        \caption{$p=-2.0$}
        \label{fig:h2h-p2}
    \end{subfigure}\hfill
    \begin{subfigure}[b]{0.25\textwidth}
        \centering
        \includegraphics[width=\textwidth]{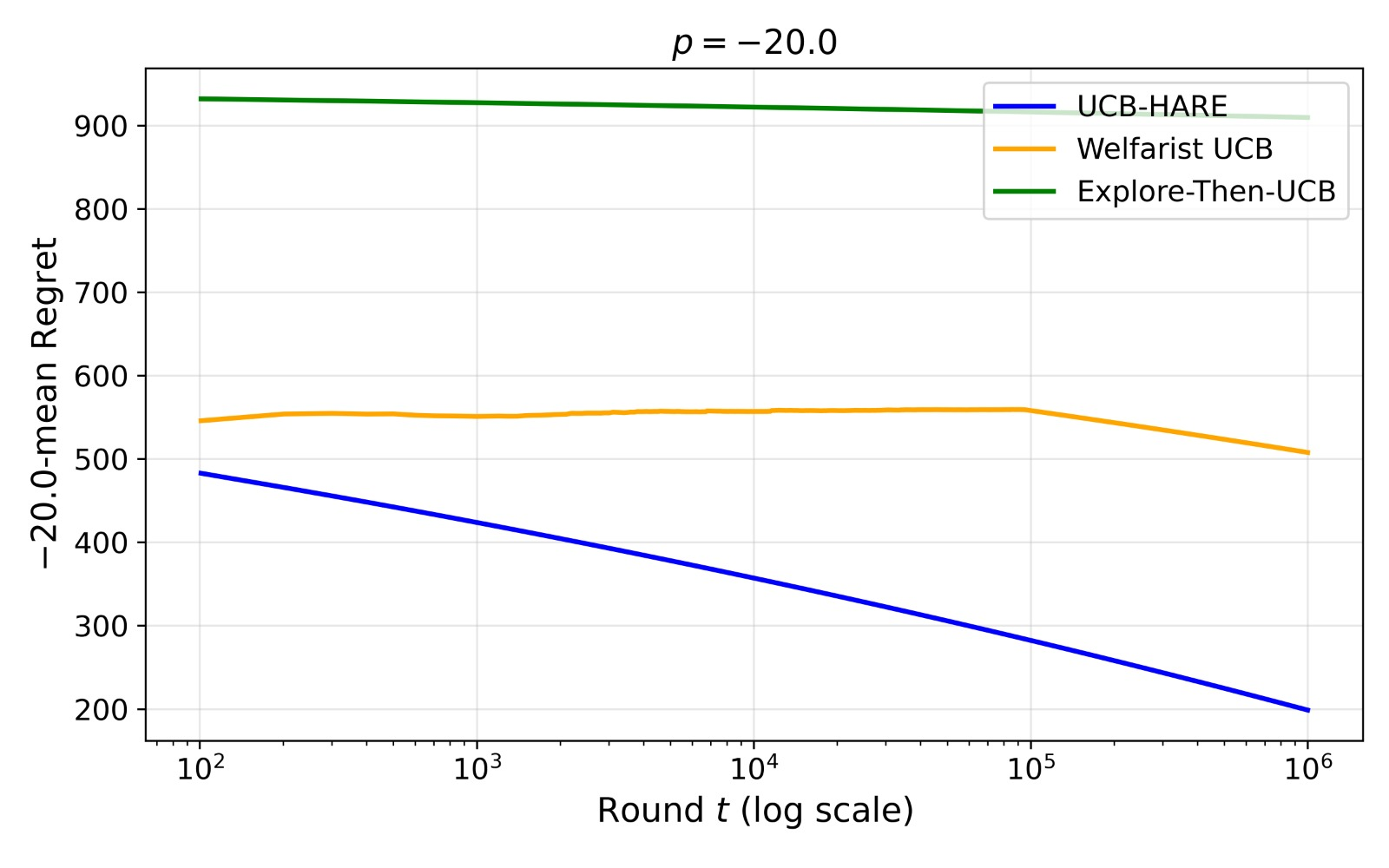}
        \caption{$p=-20$}
        \label{fig:h2h-p20}
    \end{subfigure}\hfill
    \begin{subfigure}[b]{0.25\textwidth}
        \centering
        \includegraphics[width=\textwidth]{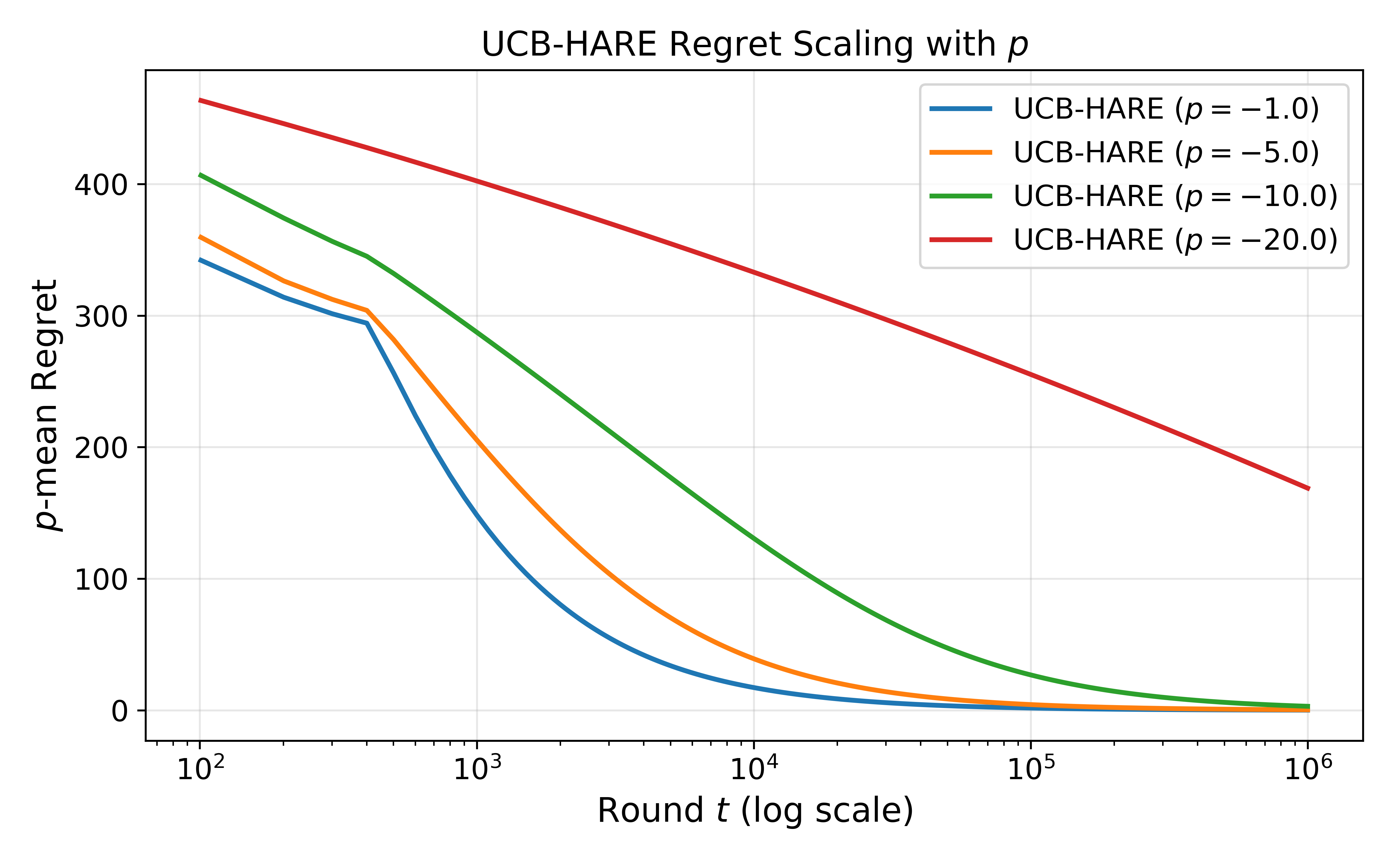}
        \caption{Varying $p$}
        \label{fig:q-sweep}
    \end{subfigure}
    \vspace{-1mm}
    \caption{(a -- c) Regret comparison of \hare\ against \textsf{Welfarist-UCB} and \textsf{Explore-Then-UCB}. (d) \hare's regret across $p \in \{-1,-5,-10,-20\}$.}
    \label{fig:expt-abc}

    \vspace{2em}
    
    \centering
    \includegraphics[width=\columnwidth]{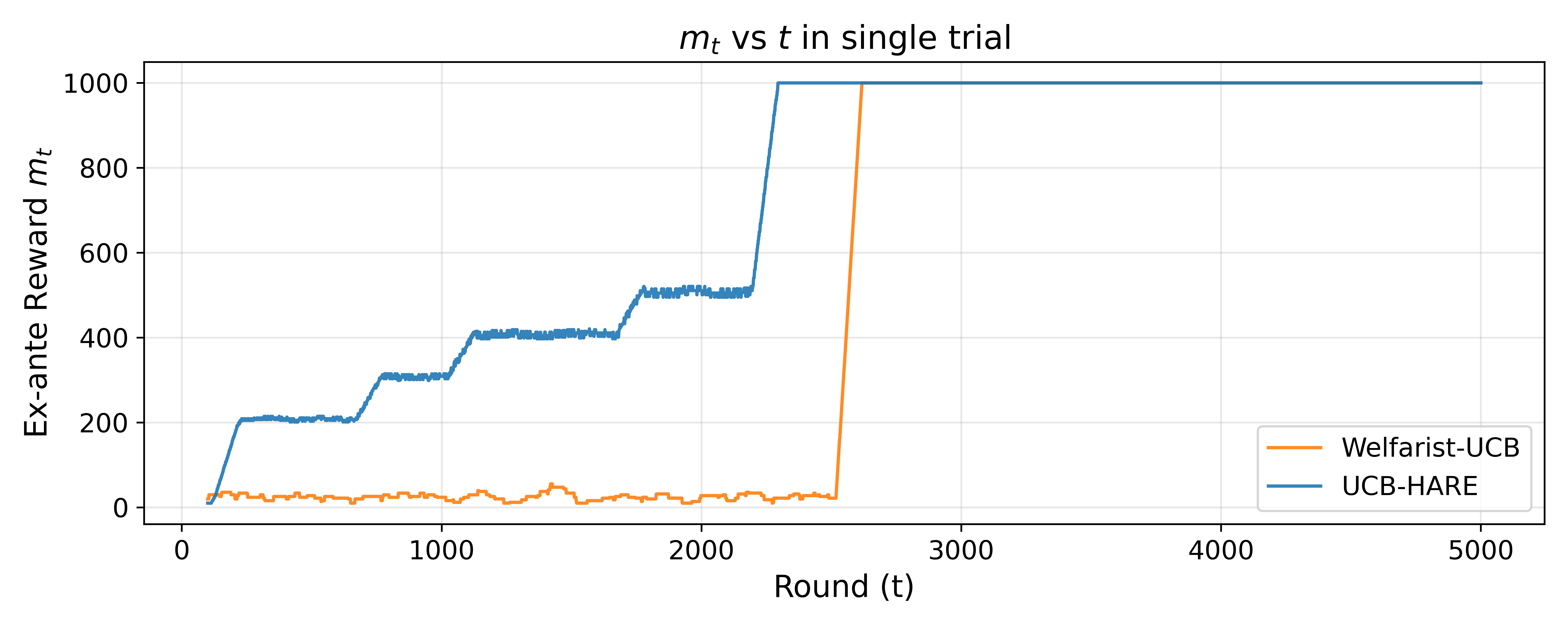}
    \caption{Single-trial ex-ante reward $m_t$ on a $200$-arm instance with five elevated-mean arms ($\mustar = 1000$). \hare 's staircase reveals the advantage of successive \textit{anchor} discoveries - establishing a welfare floor early on; \textsf{Welfarist-UCB} remains at the uniform exploration baseline throughout Phase I before jumping to the optimum.}
    \label{fig:expt-c}
\end{figure*}



\paragraph{Experiment C: Ex-ante reward vs $t$.} We construct an instance with $k=200$ arms, five of which have means $200$, $400$, $600$, $800$, $1000$ and the rest with mean $10$, and run a single trial (no averaging across seeds) over $T=5000$ rounds at $p = -q = -2.0$, tracking the ex-ante reward $m_t$. Single-trial outcomes vary with the random permutation drawn by the algorithm; the run shown in \cref{fig:expt-c} is representative and captures the novelty of our \textit{harmonic anchored rank exploration}. We highlight the following:

\begin{enumerate}[align=left, leftmargin=*]
    \item \textit{Welfare improves stepwise, not in one jump.} \hare's roundwise ex-ante reward $m_t$ increases in a sequence of steps, each corresponding to a newly certified anchor, before reaching the optimum. This matches the bound $m_t \ge c_0 \mu_\star \max(1/k, a_b/k)$ of \cref{lem:ex-ante-reward}, which increases in step with the expanding frontier $a_b$.



    \item \textit{Uniform exploration is sub-optimal.} \textsf{Welfarist-UCB}'s $m_t$ sits at the uniform-exploration floor ($m_t \approx 25 \gtrsim \mu_\star/k$) for the entire duration of Phase I, confirming that this baseline is indeed the source of the $k^{(q+1)/2}$ penalty it incurs. This stems from the fact that \textsf{Welfarist-UCB} samples all $k$ arms comparably before committing to UCB, regardless of which is optimal.
\end{enumerate}

\textbf{Experiment D: Regret scaling with $k$ and $q$.} Here, we choose an instance where one arm has mean $\mu = 1000$, a block of near-optimal arms have mean $\mu = 950$, and the remaining arms have mean $\mu = 10$, with arm identities randomly permuted. We compare $k = 10$ against $k = 100$ over horizons up to $T = 10^6$, averaged over $50$ independent trials, at three fairness levels $p \in \{-1.0, -2.0, -5.0\}$ with $\sigma^2 = 1.0$ (Figure~\ref{fig:regret-k-q}).

\begin{figure}[hbpt]
    \centering
    \includegraphics[width=0.32\textwidth]{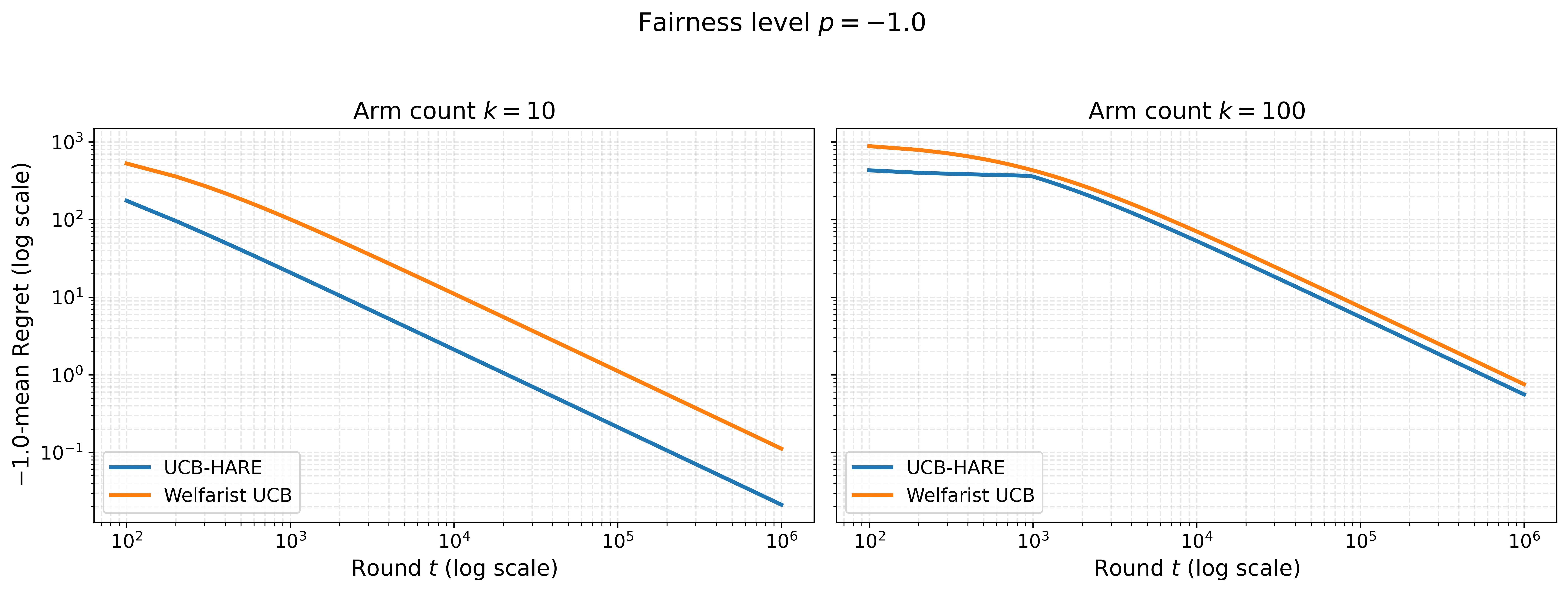}
    \includegraphics[width=0.32\textwidth]{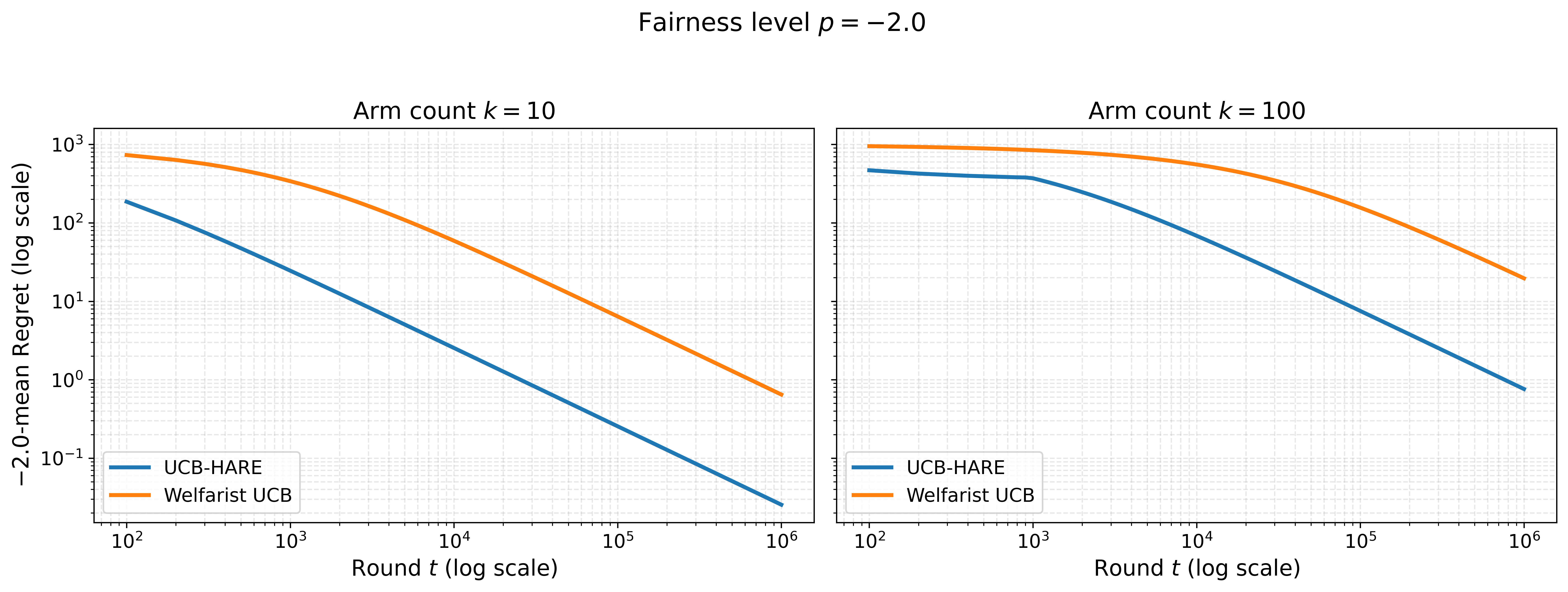}
    \includegraphics[width=0.32\textwidth]{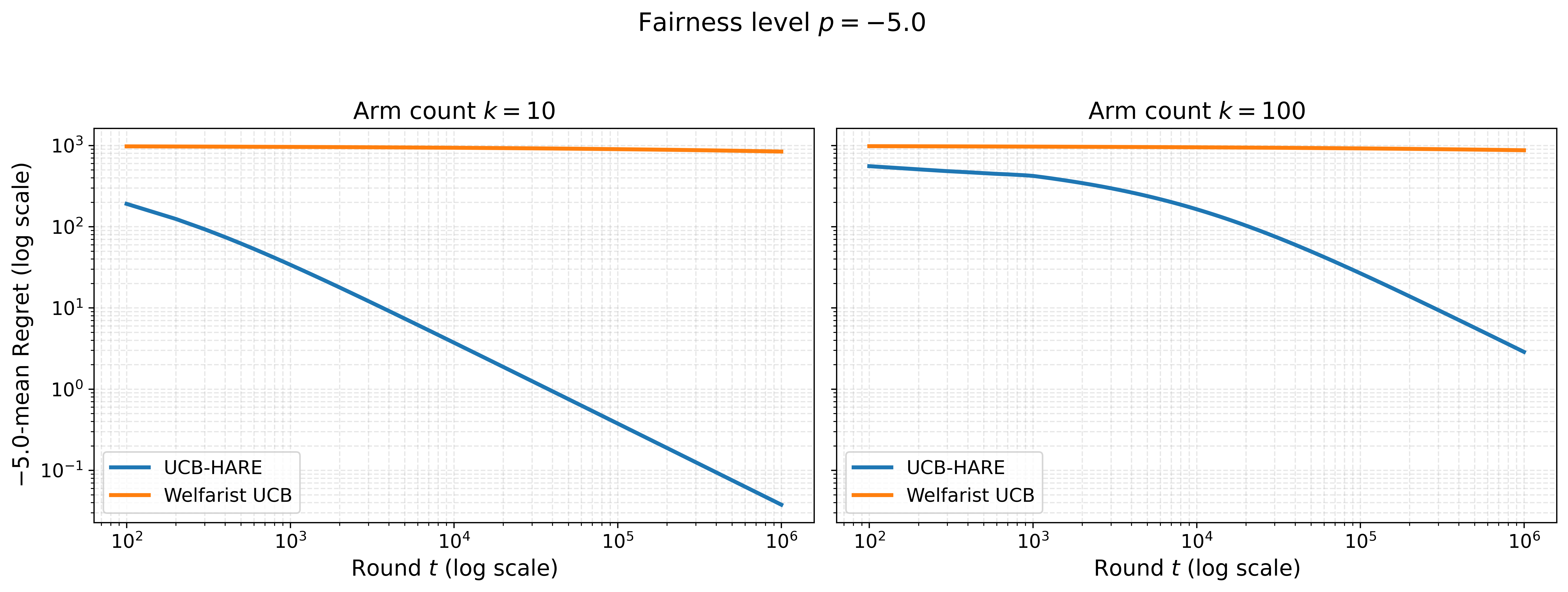}
    \caption{Regret comparison between \hare\ and \textsf{Welfarist-UCB} at $k \in \{10, 100\}$, for
    fairness levels $p \in \{-1.0, -2.0, -5.0\}$, in
    Experiment D.}
    \label{fig:regret-k-q}
\end{figure}

The gap between the two algorithms grows with $q = -p$. At $q = 1$, \hare\ and \textsf{Welfarist-UCB} stay close at both arm counts, matching the regime $q \leq 1$, where both rates scale as $\sqrt{k/T}$. As $q$ rises to $2$ and then $5$, \textsf{Welfarist-UCB}'s regret tends to flatten at its starting value for \textit{both} $k = 10$ and $100$, while \hare\ retains its performance better. This matches the gap between $k^{(q+1)/2}$ and $k^{q/2}$ from \Cref{cor:upper-regimes}: uniform exploration gets more costly as $q$ grows, while \hare\ 's \textit{harmonic schedule} avoids this extra cost. 

Pushing $q$ further, to $10$ and $20$, makes this flattening pronounced for both
algorithms (\Cref{fig:regret-rawl}). This reflects the strongly fair regime's convergence
to the Rawlsian maximin ideal: as $q \to \infty$, the $(-q)$-mean is dominated by the
smallest ex-ante reward in the sequence, so both algorithms are limited by their worst
round. \textsf{Welfarist-UCB} reaches this limit early, at the uniform-exploration floor
$\mustar/k$, while \hare's harmonic schedule certifies a positive-mean anchor sooner and
sustains improvement over a longer horizon.

\begin{figure}[hbpt]
    \centering
    \includegraphics[width=0.48\textwidth]{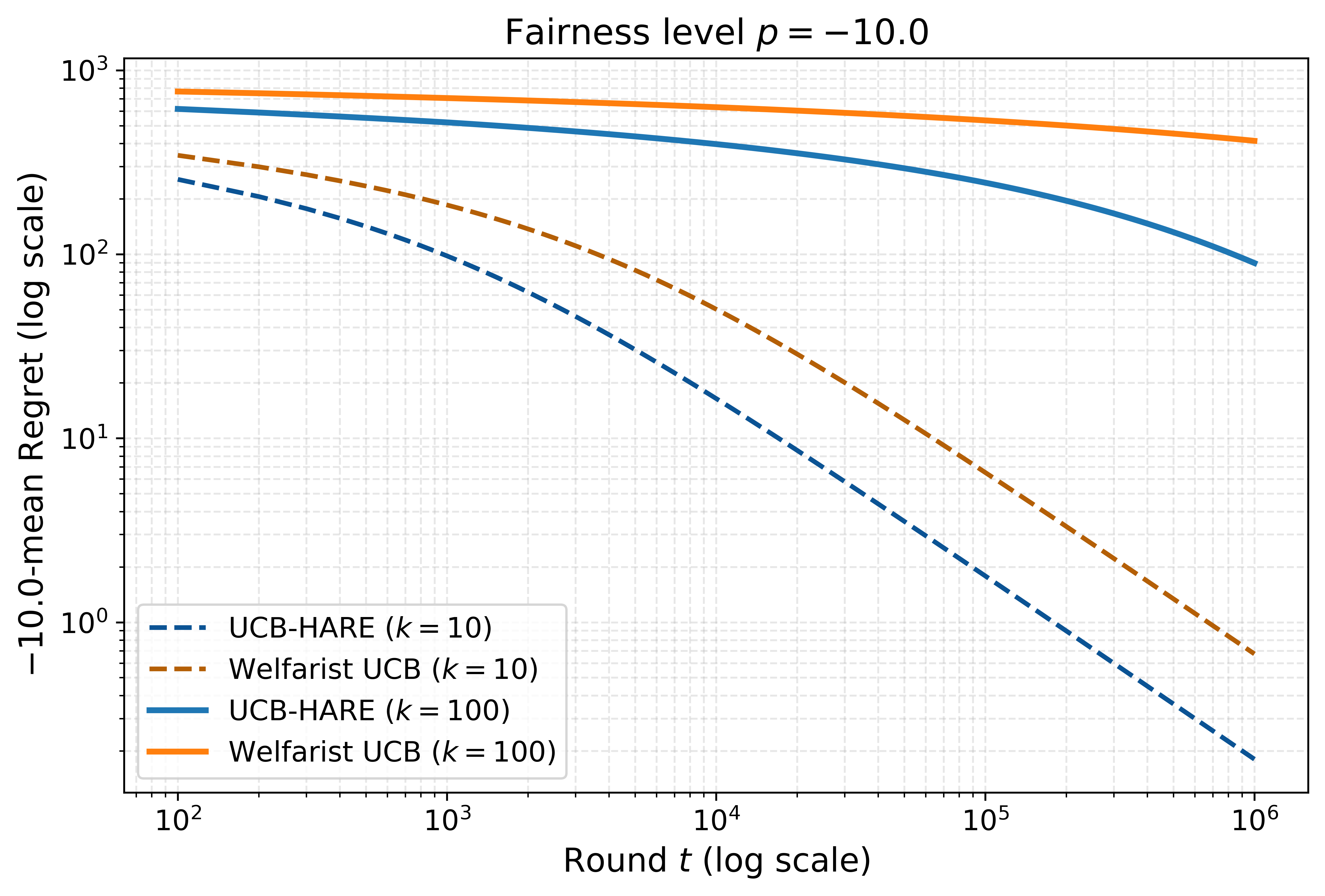}
    \includegraphics[width=0.48\textwidth]{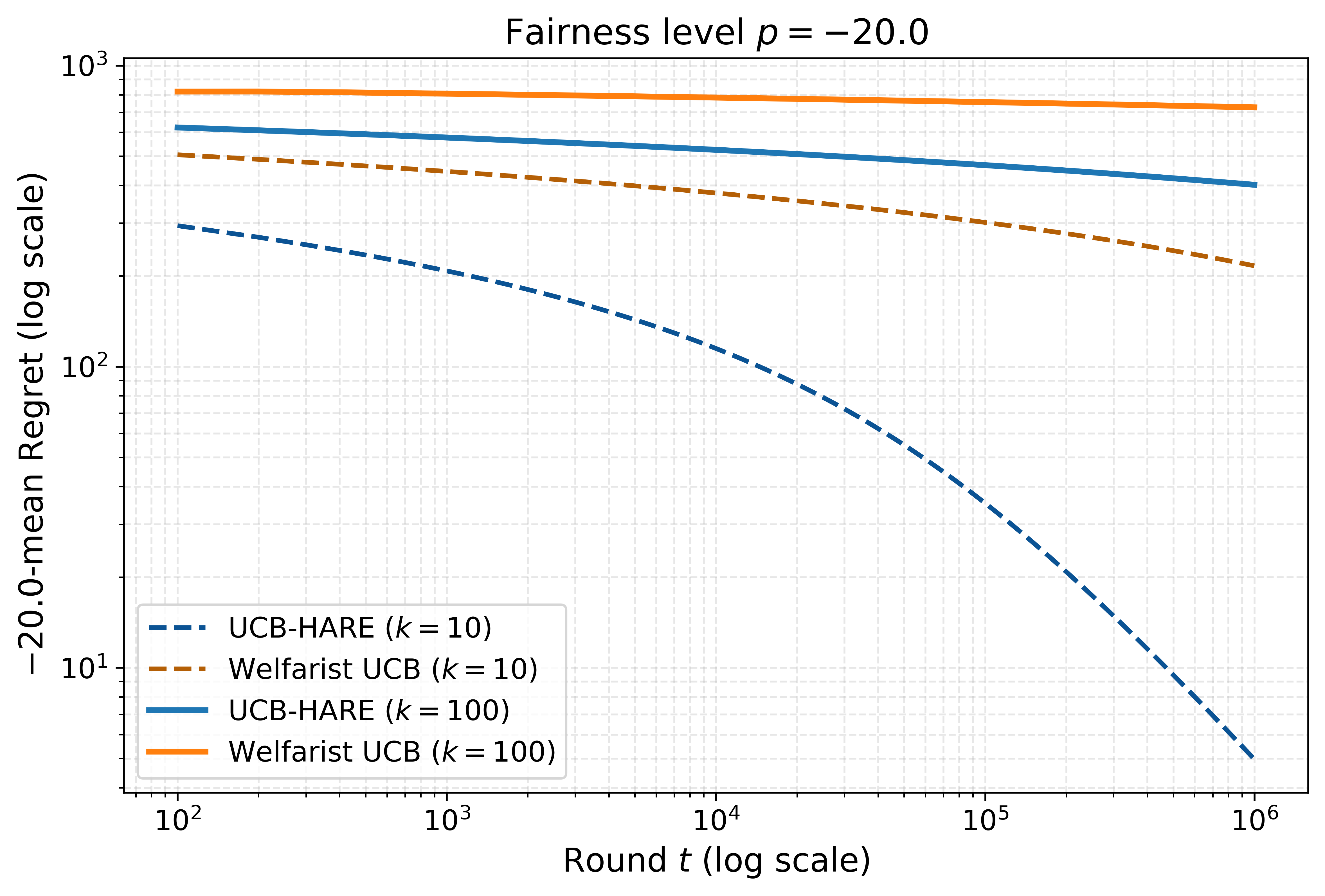}
    \caption{At large $q$, regrets of both algorithms flatten toward the Rawlsian
    worst-round bottleneck, with \hare\ sustaining improvement over a longer horizon.}
    \label{fig:regret-rawl}
\end{figure}

\clearpage
\printbibliography

\end{document}